\documentclass[twoside]{article}

\usepackage[accepted]{aistats2021}
\usepackage{microtype}
\usepackage{graphicx}
\usepackage{booktabs} 
\usepackage{amsfonts}       
\usepackage{nicefrac}       
\usepackage{microtype}      
\usepackage{amsthm}
\usepackage{amsmath}
\usepackage{amssymb}
\usepackage{mathtools}
\mathtoolsset{showonlyrefs=true}
\usepackage{wrapfig}
\usepackage[subrefformat=parens]{subcaption}

\usepackage{hyperref}

\usepackage{algorithm}
\usepackage{algorithmic}

\newtheorem{thm}{Theorem}

\newtheorem{prop}{Proposition}
\newtheorem{cor}{Corollary}
\newtheorem{lem}{Lemma}
\newtheorem{asm}{Assumption}

\theoremstyle{definition}

\newtheorem*{rmk}{Remark}

\newcommand{\defeq}{\overset{\mathrm{def}}{=}}
\DeclareMathOperator*{\argmin}{arg\,min}

\DeclareMathOperator*{\sgn}{sgn}

\newcommand{\Sp}{\mathcal{H}}
\newcommand{\Rfsp}{\mathcal{H}_M}
\newcommand{\Cbsp}{\mathcal{H}^{+}_M}
\newcommand{\Ltwo}[1]{L^2(#1)}
\newcommand{\G}{g_\lambda}
\newcommand{\GG}{g_{M, \lambda}}
\newcommand{\GGG}{\widetilde{g}_\lambda}
\newcommand{\GGGG}{\widetilde{g}_{M, \lambda}}

\newcommand{\loss}{\mathcal{L}}
\usepackage[round]{natbib}

\begin{document}

%

%

\twocolumn[

\runningtitle{Exponential Convergence Rates of Classification Errors on Learning with SGD and Random Features}
\aistatstitle{Exponential Convergence Rates of Classification Errors \\on Learning with SGD and Random Features}

\runningauthor{Shingo Yashima, Atsushi Nitanda, Taiji Suzuki}

\aistatsauthor{ Shingo Yashima$^{1}$ \\ \texttt{syashima9@gmail.com} \And Atsushi Nitanda$^{1,2,3}$\\ \texttt{nitanda@mist.i.u-tokyo.ac.jp} \And  Taiji Suzuki$^{1,2}$\\ \texttt{taiji@mist.i.u-tokyo.ac.jp}}

\aistatsaddress{ $^1$Graduate School of Information Science and Technology, The University of Tokyo\\
$^2$Center for Advanced Intelligence Project, RIKEN\\
$^3$PRESTO, Japan Science and Technology Agency } ]

\begin{abstract}
Although kernel methods are widely used in many learning problems, they have poor scalability to large datasets. To address this problem, sketching and stochastic gradient methods are the most commonly used techniques to derive computationally efficient learning algorithms. We consider solving a binary classification problem using random features and stochastic gradient descent, both of which are common and widely used in practical large-scale problems. Although there are plenty of previous works investigating the efficiency of these algorithms in terms of the convergence of the objective loss function, these results suggest that the computational gain comes at expense of the learning accuracy when dealing with general Lipschitz loss functions such as logistic loss. In this study, we analyze the properties of these algorithms in terms of the convergence not of the loss function, but the classification error under the strong low-noise condition, which reflects a realistic property of real-world datasets. We extend previous studies on SGD to a random features setting, examining a novel analysis about the error induced by the approximation of random features in terms of the distance between the generated hypothesis to show that an exponential convergence of the expected classification error is achieved even if random features approximation is applied. We demonstrate that the convergence rate does not depend on the number of features and there is a significant computational benefit in using random features in classification problems under the strong low-noise condition.
\end{abstract}

\section{Introduction}
Kernel methods are commonly used to solve a wide range of problems in machine learning, as they provide flexible non-parametric modeling techniques and come with well-established theories about their statistical properties \citep{caponnetto2007optimal, steinwart2009optimal, mendelson2010regularization}. However, computing estimators in kernel methods can be prohibitively expensive in terms of memory requirements for large datasets.
There are two popular approaches to scaling up kernel methods. The first is sketching, which reduces data-dimensionality by random projections. A random features method \citep{rahimi2008random} is a representative, which approximates a reproducing kernel Hilbert space (RKHS) by a finite-dimensional space in a data-independent manner. The second is stochastic gradient descent (SGD), which allows data points to be processed individually in each iteration to calculate gradients. Both of these methods are quite effective in reducing memory requirements and are widely used in practical tasks.\\
For the theoretical properties of random features, several studies have investigated the approximation quality of kernel functions \citep{sriperumbudur2015optimal, sutherland2015error, pmlr-v89-szabo19a}, but only a few have considered the generalization properties of learning with random features. For the regression problem, its generalization properties in ERM and SGD settings, respectively, have been studied extensively by \citet{rudi2017generalization} and \citet{carratino2018learning}. In particular, they showed that $O(\sqrt{n} \log n)$ features are sufficient to achieve the usual $O(1/\sqrt{n})$ learning rate, indicating that there is a computational benefit to using random features. 
However, it remains unclear whether or not it is computationally efficient for other tasks. By \citet{rahimi2009weighted}, the generalization properties were studied with Lipschitz loss functions under $\ell_\infty$-constraint in hypothesis space, and it was shown that $O(n\log n)$ features are required for $O(1/\sqrt{n})$ learning bounds. Also, by \citet{li2018toward}, learning with Lipschitz loss and standard regularization was considered instead of $\ell_\infty$-constraint, and similar results were attained. Both results suggest that computational gains come at the expense of learning accuracy if one considers general loss functions.
\\
In this study, learning classification problems with random features and SGD are considered, and the generalization property is analyzed in terms of the \textit{classification error}. Recently, it was shown that the convergence rate of the excess classification error can be made exponentially faster by assuming the \textit{strong low-noise condition} \citep{tsybakov2004optimal, koltchinskii2005exponential} that conditional label probabilities are uniformly bounded away from $1/2$ \citep{pillaud2018exponential, nitanda2018stochastic}. We extend these analyses to a random features setting to show that the exponential convergence is achieved if a sufficient number of features are sampled. Unlike when considering the convergence of loss function, the resulting convergence rate of the classification error is independent of the number of features. In other words, an arbitrary small classification error is achievable as long as there is a sufficient number of random features. So our result suggests that there is indeed a computational benefit to use random features in classification problems under the strong low-noise condition.
\paragraph{Remark}
Although several studies consider the optimal sampling distributions of features in terms of the worst-case error and show the superiority of random features \citep{bach2017equivalence, rudi2017generalization, li2018toward, sun2018but}, we do not explore this direction and treat the original random features because these distributions are generally intractable or require much computational cost to sample \citep{bach2017equivalence} whereas an efficient sampling algorithm is proposed in the case of Gaussian kernel \citep{avron2017random}. 
\\In addition, we should refer to Nystr\"om method \citep{williams2001using}, which is also a popular method to scale up kernel methods. In contrast to random features, Nystr\"om method approximates kernel function in data-dependent way. As a result, similar to calculating an optimized sampling distribution on random features, Nystr\"om method also requires data points before actual training starts and needs $O(nM)$ memory, which is more expensive than $O(M)$ in random features. These are reasons why we dealt with original algorithm of random features in this study.
\paragraph{Our Contributions} Our contributions are twofold. First, we analyze the error induced by the approximation of random features in terms of the $L^\infty$-norm between the generated hypothesis including population risk minimizers and empirical risk minimizers when using general Lipschitz loss functions in Section \ref{errany}. Our results can be framed as an extension of the analysis of \citet{cortes2010impact, sutherland2015error}, which analyzed the error in terms of the distance between empirical risk minimizers when using a hinge loss. However, it is not straightforward to extend these results to our case since we cannot access the closed-form solutions, unlike those previous results, when using the general loss functions and treating population risk minimizers. In addition, since the true and the approximated minimizer lie in different function spaces, it is not easy to derive $L^\infty$-norm bound between them. We deal with these difficulties with novel proof techniques.
\\
Second, using the above result, we prove that the exponential convergence rate of the excess classification error under the strong low-noise condition is achieved if a sufficient number of features are sampled in Section \ref{mainres}. Then we show that there is a significant computational gain in using random features rather than a full kernel method for obtaining a relatively small classification error. We also validate these results through experiments on synthetic datasets in Section \ref{exp}.

\section{Problem Setting}
\label{probset}
\subsection{Binary Classification Problem}
Let $\mathcal{X}$ and $\mathcal{Y}=\{-1, 1\}$ be a feature space and the set of binary labels, respectively; $\rho$ denotes a probability measure on $\mathcal{X} \times \mathcal{Y}$, by $\rho_{\mathcal{X}}$ the marginal distribution on $X$, and by $\rho(\cdot | X)$ the conditional distribution on $Y$, where $(X, Y) \sim \rho$. In general, for a probability measure $\mu$, $\Ltwo{\mu}$ denotes a space of square-integrable functions with respect to $\mu$, and $L^2(\mathcal{X})$ denotes one with respect to the Lebesgue measure. Similarly, $L^\infty(\mu)$ denotes a space of functions for which the essential supremum with respect to $\mu$ is bounded, and $L^\infty(\mathcal{X})$ denotes one with respect to Lebesgue measure.\\
In the classification problem, our final objective is to choose a discriminant function $g:\mathcal{X} \rightarrow \mathbb{R}$ such that the sign of $g(X)$ is an accurate prediction of $Y$. Therefore, we intend to minimize the expected classification error $\mathcal{R}(g)$ defined below amongst all measurable functions:
\begin{equation}
    \mathcal{R}(g) \defeq \mathbb{E}_{(X, Y)\sim \rho} \left[I(\sgn(g(X)), Y )\right],\label{error}
\end{equation}
where $\sgn(x) = 1$ if $x>0$ and $-1$ otherwise, and $I$ represents 0-1 loss:
\begin{align}
    I(y, y^\prime) \defeq \begin{cases}
1 & (y \neq y^\prime) \\
0 & (y = y^\prime).
\end{cases}
\end{align}
By definition, $g(x) = \mathbb{E}[Y|x] = 2\rho(1|x) - 1$ minimizes $\mathcal{R}$. However, directly minimizing \eqref{error} to obtain the Bayes classifier is intractable because of its non-convexity. Thus, we generally use the convex surrogate loss $l(\zeta, y)$ instead of the 0-1 loss and minimize the expected loss function $\mathcal{L}(g)$ of $l$:
\begin{align}
    \mathcal{L}(g) \defeq \mathbb{E}_{(X, Y)\sim \rho} \left[ l(g(X), Y )\right]. \label{loss}
\end{align}
In general, the loss function $l$ has a form $l(\zeta, y) = \phi(\zeta y)$ where $\phi:\mathbb{R} \rightarrow \mathbb{R}$ is a non-negative convex function. The typical examples are logistic loss, where $\phi(v) = \log(1+\exp(-v))$ and hinge loss, where $\phi(v) = \max \{0, 1-v\}.$ Minimizing the expected loss function \eqref{loss} ensures minimizing the expected classification \eqref{error} if $l$ is \textit{classification-calibrated} \citep{bartlett2006convexity}, which has been proven for several practically implemented losses including hinge loss and logistic loss.
\subsection{Kernel Methods and Random Features}
In this study, we consider a reproducing kernel Hilbert space (RKHS) $\mathcal{H}$ associated with a positive definite kernel function $k:\mathcal{X} \times \mathcal{X} \rightarrow \mathbb{R}$ as the hypothesis space. It is known \citep{aronszajn1950theory} that a positive definite kernel $k$ uniquely defines its RKHS $\Sp$ such that the reproducing property $f(x) = \langle f, k(\cdot, x) \rangle_\Sp$ holds for all $f \in \Sp$ and $x \in \mathcal{X}$, where $\langle \cdot, \cdot \rangle_\Sp$ denotes the inner product of $\Sp$. Let $\|\cdot\|_\Sp$ denote the norm of $\Sp$ induced by the inner product. Under these settings, we attempt to solve the following minimization problem:
\begin{align}
    \min_{g\in \Sp} \ \mathcal{L}(g) + \frac{\lambda}{2} \|g \|_\Sp^2 \label{problem}
\end{align}
where $\lambda>0$ is a regularization parameter. 
\\
However, because solving the original problem \eqref{problem} is usually computationally inefficient for large-scale datasets, the approximation method is applied in practice.
Random features \citep{rahimi2008random} is a widely used method for scaling up kernel methods because of its simplicity and ease of implementation. Additionally, it approximates the kernel in a data-independent manner, making it easy to combine with SGD. In random features, a kernel function $k$ is assumed to have the following expansion in some space $\Omega$ with a probability measure $\tau$:
\begin{equation}
    k(x, y) = \int_\Omega \varphi(x, \omega) \overline{\varphi(y, \omega)} \mathrm{d}\tau(\omega). 
    \label{eq1}
\end{equation}
The main idea behind random features is to approximate the integral \eqref{eq1} by its Monte-Carlo estimate:
\begin{equation}
    k_M(x, y) \defeq \frac{1}{M} \sum_{i=1}^{M} \varphi(x, \omega_i) \overline{\varphi(y, \omega_i)}, \ \ \  \omega_i \overset{i.i.d.}{\sim} \tau.\label{eq2}
\end{equation}
For example, if $k$ is a shift invariant kernel, by Bochner's theorem \citep{yoshida1995functional}, the expansion \eqref{eq1} is achieved with $\varphi(x, \omega) = C^\prime e^{i \omega^\top x}$, where $C^\prime$ is a normalization constant. Then, the approximation \eqref{eq2} is called random Fourier features \citep{rahimi2008random}, which is the most widely used variant of random features.\\
We denote the RKHS associated with $k$ and $k_M$ by $\Sp$ and $\Rfsp$, respectively. These spaces then admit the following explicit representation \citep{bach2017equivalence, bach2017breaking}:
\begin{align}
        &\Sp = \left \{\int_\Omega \beta(\omega) \varphi(\cdot, \omega) \mathrm{d} \tau(\omega) \ \middle| \ 
    \beta \in L^2(\tau) \right \}\\
    &\Rfsp = \left \{\sum_{i=1}^M \frac{\beta_i}{\sqrt{M}} \varphi(\cdot, \omega_i)  \ \middle| \ 
    | \beta_i | < \infty \right \}.
\end{align}
We note that the approximation space $\Rfsp$ is not necessarily contained in the original space $\Sp$. For $g \in \Sp$ and $h \in \Rfsp$, the following RKHS norm relations hold:
\begin{align}
    &\|g\|_{\Sp} = \inf \left\{\| \beta \|_{L^2(\tau)} \ \middle|\  g = \int_\Omega \beta(\omega) \varphi(\cdot, \omega) \mathrm{d} \tau(\omega) \right\}\\
    &\|h\|_{\Rfsp} = \inf \left\{ \|\beta\|_2 \ \middle|\  h = \sum_{i=1}^M \frac{\beta_i}{\sqrt{M}} \varphi(\cdot, \omega_i) \right\}.
\end{align}
As a result, the problem \eqref{problem} in the approximation space $\Rfsp$ is reduced to the following generalized linear model:
\begin{align}
    \min_{\beta \in \mathbb{R}^M} \ \mathcal{L}(\beta^\top \phi_M) + \frac{\lambda}{2} \| \beta \|_2^2 \label{rfproblem}
\end{align}
where $\phi_M$ is a feature vector:
\begin{align}
    \phi_M \defeq \frac{1}{\sqrt{M}}[ \varphi(\cdot, \omega_1), \ldots, \varphi(\cdot, \omega_M)]^\top.
\end{align}
In this paper, we consider solving the problem \eqref{rfproblem} using the averaged SGD. The details are discussed in the following section.

\subsection{Averaged Stochastic Gradient Descent}
SGD is the most popular method to solve large scale learning problems. In this section, we discuss a specific form of an optimization procedure. For the minimization problem \eqref{rfproblem}, its gradient with respect to $\beta$ is given as follows:
\begin{equation}
    \mathbb{E}\left[ \partial_\zeta l(\beta^\top \phi_M(X), Y) \phi_M(X) + \lambda \beta \right],
\end{equation}
where $\partial_\zeta$ is a partial derivative with respect to the first variable of $l$. Thus, the stochastic gradient with respect to $\beta$ is given by $\partial_\zeta l(\beta^\top \phi_M(X), Y) \phi_M(X) + \lambda \beta$. We note that the update on the $\beta$ parameter corresponds to the update on the function space $\Rfsp$, because a gradient on $\Rfsp$ is given by
\begin{equation}
    \mathbb{E}\left[ \partial_\zeta l(\beta^\top \phi_M(X), Y) \phi_M(X) + \lambda \beta \right]^\top\phi_M.
\end{equation}
We consider the averaged variants of SGD, since it is widely known that gradient averaging gives faster convergence than plain SGD on strongly convex problems \citep{lacoste2012simpler}. The algorithm of random features and averaged SGD is described in Algorithm \ref{alg1}.
\begin{algorithm}
\caption{Random Feature + SGD}         
\label{alg1}      
\begin{algorithmic}
    \REQUIRE{number of features $M$, regularization parameter $\lambda$, number of iterations $T$, learning rates $\{\eta_t\}_{t=1}^T$}, averaging weights $\{\alpha_t\}_{t=1}^{T+1}$
    \ENSURE{classifier $\overline{g}_{T+1}$}
    \STATE Randomly draw feature variables $\omega_1, \ldots, \omega_M \sim \tau$
        \STATE Initialize $\beta_1 \in \mathbb{R}^M$
        \FOR{$t = 1, \dots, T$}
        \STATE Randomly draw samples $(x_t, y_t) \sim \rho$
        \STATE $\beta_{t+1} \leftarrow \beta_t - \eta_t \left( \partial_\zeta l(\beta_{t}^\top \phi_M(x_t), y_t) \phi_M(x_t) + \lambda \beta_t \right)$
        \ENDFOR
        \STATE $\overline{\beta}_{T+1} = \sum_{t=1}^{T+1}\alpha_t \beta_{t}$
        \STATE \textbf{return} $\overline{g}_{T+1} = \overline{\beta}_{T+1}^\top \phi_M$ 
\end{algorithmic}
\end{algorithm}
Following \citet{nitanda2018stochastic}, we set the learning rate and the averaging weight as follows:
\begin{equation}
    \eta_t = \frac{2}{\lambda(\gamma+t)}, \ \ \alpha_t = \frac{2(\gamma+t-1)}{(2\gamma+T)(T+1)},
\end{equation}
where $\gamma$ is an offset parameter for the time index. We note that an averaged iterate $\overline{\beta}_t$ can be updated iteratively as follows:
\begin{align}
    &\overline{\beta}_1 = \beta_1, \\
    &\overline{\beta}_{t+1} = (1-\theta_t) \overline{\beta}_{t} + \theta_t \beta_{t+1}, \ \ \theta_t = \frac{2(\gamma+t)}{(t+1)(2\gamma+t)}.
\end{align}
Using this formula, we can compute the averaged output without storing all internal iterate $(\beta_t)_{t=1}^{T+1}.$
\paragraph{Computational Complexity}
\label{comp}
If we assume the evaluation of a feature map $\varphi(x, \omega)$ to have a constant cost, one iteration in Algorithm \ref{alg1} requires $O(M)$ operations. As a result, one pass SGD on $n$ samples requires $O(Mn)$ computational time. On the other hand, the full kernel method without approximation requires $O(n)$ computations per iteration; thus, the overall computation time is $O(n^2)$, which is much more expensive than random features. 
For the memory requirements, random features needs to store $M$ coefficients, and it does not depend on the sample size $n$. On the other hand, we have to store $n$ coefficients in the full kernel method, so it is also advantageous to use random features in large-scale learning problems.

\section{Error Analysis of Random Features}
\label{errany}
Our primary purpose here is to bound the distance between the hypothesis generated by solving the problems in each space $\Sp$ and $\Rfsp$.
Population risk minimizers in spaces $\Sp, \Rfsp$ are defined as below:
\begin{align}
   \G &= \argmin_{g \in \Sp} \left(\loss(g) + \frac{\lambda}{2} \|g\|_{\Sp}^2 \right), \\
\GG &= \argmin_{g \in \Rfsp} \left(\loss(g) + \frac{\lambda}{2} \|g\|_{\Rfsp}^2 \right).
\end{align}
The uniqueness of minimizers is guaranteed by the regularization term. \\
First, the $L^\infty(\rho_{\mathcal{X}})$-norm is bound between $g_\lambda$ and $g_{M, \lambda}$ when the loss function $l(\cdot, y)$ is Lipschitz continuous. Then, a more concrete analysis is provided when $k$ is a Gaussian kernel.

\subsection{Error Analysis for Population Risk Minimizers}
Before beginning the error analysis, some assumptions about the loss function and kernel function are imposed.
\begin{asm}
\label{as2-1} $l(\cdot, y)$ is convex and $G$-Lipscitz continuous, that is, there exists $G > 0$ such that for any $\zeta, \zeta^\prime \in \mathbb{R}$ and $y \in \mathcal{Y}$,
    \begin{equation}
        | l(\zeta, y) - l(\zeta^\prime, y) | \leq G|\zeta - \zeta^\prime|.
    \end{equation}
\end{asm}

This assumption implies $G$-Lipschitzness of $\mathcal{L}$ with respect to the $\Ltwo{\rho_\mathcal{X}}$ norm, because 
\begin{align}
|\mathcal{L}(g) - \mathcal{L}(h)| &\leq G \int |g(x) - h(x)| \mathrm{d}\rho_{\mathcal{X}}(x) \\
&\leq G\| g - h\|_{\Ltwo{\rho_\mathcal{X}}}
\end{align}
for any $g, h \in \Ltwo{\rho_\mathcal{X}}$. For several practically used losses, such as logistic loss or hinge loss, this assumption is satisfied with $G=1$.\\
To control continuity and boundedness of the induced kernel, the following assumptions are required:
\begin{asm}
\label{as2-3} The function $\varphi$ is continuous and there exists $R > 0$ such that $|\varphi(x, \omega)| \leq R$ for any $x \in \mathcal{X}, \omega \in \Omega$.
\end{asm}
If $k$ is Gaussian and $\varphi$ is its random Fourier features, it is satisfied with $R=1$. This assumption implies $\sup_{x, y \in \mathcal{X}}k(x, y) \leq R^2, \sup_{x, y \in \mathcal{X}}k_M(x, y) \leq R^2$ and it leads to an important relationship $R\| \cdot \|_{\Sp} \geq \| \cdot \|_{L^\infty(\mathcal{X})}, R\| \cdot \|_{\Rfsp} \geq \| \cdot \|_{L^\infty(\mathcal{X})}.$\\
For the two given kernels $k$ and $k_M$, $k+k_M$ is also a positive definite kernel, and its RKHS includes $\Sp$ and $\Rfsp$. The last assumption imposes a specific norm relationship in its combined RKHS of $\Sp$ and $\Rfsp$.  
\begin{asm}
\label{as2-2} Let $\Cbsp$ be RKHS with the kernel function $k+k_M$. Then there exists $0\leq p<1$, and a constant $C(\delta)>0$ depends on $0<\delta \leq 1$ that satisfies, for any $f \in \Cbsp$,
    \begin{equation}
        \|f\|_{L^\infty(\rho_{\mathcal{X}})} \leq C(\delta) \|f\|_{\Cbsp}^{p} \|f\|_{\Ltwo{\rho_\mathcal{X}}}^{1-p}
    \end{equation}
    with probability at least $1-\delta.$
\end{asm}
For a fixed kernel function, the Assumption \ref{as2-2} is a commonly used condition in analysis of kernel methods \citep{steinwart2009optimal, mendelson2010regularization}. It is satisfied, for example, when the eigenfunctions of the kernel are uniformly bounded and the eigenvalues $\{\mu_i\}_i$ decay at the rate $i^{-1/p}$ \citep{mendelson2010regularization}. In Theorem \ref{thm2}, specific $p$ and $C(\delta)$ that satisfy the condition for the case of a Gaussian kernel and its random Fourier features approximation are derived. \\
Here, we introduce our primary result, which bounds the distance between $\G$ and $\GG$ in terms of $L^\infty(\rho_{\mathcal{X}})$-norm. The complete statement, including proof and all constants, are found in Appendix \ref{proof1}.
\begin{thm}
\label{thm1}
Under Assumption \ref{as2-1}-\ref{as2-2}, with probability at least $1-2\delta$ with respect to the sampling of features, the following inequality holds:
\begin{align}
    &\|\G - \GG\|_{L^\infty(\rho_{\mathcal{X}})} \\&\lesssim \left( \frac{R^4 \log\frac{R}{\delta}}{M} \right)^{\min \left\{ \nicefrac{(1-p)}{4}, \nicefrac{1}{8} \right\}} \frac{C(\delta) R G^{\nicefrac{3}{4}} \|g_\lambda\|_\Sp}{\lambda^{\nicefrac{3}{4}}}.
\end{align}
\end{thm}
\par
The resulting error rate is $O(M^{-\min \left\{ \nicefrac{(1-p)}{4}, \nicefrac{1}{8} \right\}})$. It can be easily shown that a consistent error rate of $O(M^{-\nicefrac{1}{8}})$ is seen for $\Ltwo{\rho_{\mathcal{X}}}$-norm without Assumption \ref{as2-2}.

\paragraph{Comparison to Previous Results} 
The distance between empirical risk minimizers of SVM (i.e., $l$ is hinge loss) were studied in terms of the error induced by Gram matrices by \citet{cortes2010impact, sutherland2015error}. Considering $K$ and $K_M$ to be Gram matrices of kernel $k$ and $k_M$, respectively, they showed that $\|\G - \GG\|_{L^\infty(\rho_{\mathcal{X}})} \lesssim O(\|K-K_M\|_{\mathrm{op}}^{\nicefrac{1}{4}})$, where $\|\cdot\|_{\mathrm{op}}$ is an operator norm, defined in Appendix \ref{notation}. Because the Gram matrix can be considered as the integral operator on the empirical measure, we can apply Lemma \ref{lem3} and obtain $\|K-K_M\|_{\mathrm{op}} \lesssim O(M^{-\nicefrac{1}{2}})$, so the resulting rate is $O(M^{-\nicefrac{1}{8}})$. This coincides with our result, because when $\rho_{\mathcal{X}}$ is an empirical measure, Assumption 3 holds with $p = 0$. 
From this perspective, our result is an extension of these previous results, because we treat the more general Lipschitz loss function $l$ and general measure $\rho_{\mathcal{X}}$ including empirical measure. Although it is relatively easy to derive the infinite norm bound in those finite dimentional case, more careful deriviation is needed in our setting (infinite dimentional case) and our analysis is novel.\\
The case of squared loss was studied by \citet{rudi2017generalization, carratino2018learning}. In particular, in Lemma 8 of \citet{rudi2017generalization}, the $L^2$ distance between $g_\lambda$ and $g_{M, \lambda}$ is shown as $O(M^{-1/2})$ (without decreasing $\lambda$). While this is a better rate than ours, our theory covers a wider class of loss functions, and a similar phenomenon is observed in the case of empirical risk minimizers for the squared loss and hinge loss \citep{cortes2010impact}.\\
Approximation of functions in $\Sp$ by functions in $\Rfsp$ is also considered by \citet{bach2017equivalence}, but this result cannot be applied here because $g_{M, \lambda}$ is not the function closest to $g_\lambda$ in $\Rfsp$.
Finally, we note that our result cannot be obtained from the approximation analysis of Lipschitz loss functions \citep{rahimi2009weighted, li2018toward}, where the rate was shown to be $O(M^{-1/2})$ under several assumptions, because the closeness of the loss values does not imply that of the hypothesis.

\subsection{Further Analysis for Gaussian Kernels}
The following theorem shows that if $k$ is a Gaussian kernel and $k_M$ is its random Fourier features approximation, then the norm condition in Assumption \ref{as2-2} is satisfied for any $0<p<1$.
\begin{thm}
\label{thm2}

Assume $\mathrm{supp}(\rho_\mathcal{X}) \subset \mathbb{R}^d$ is a bounded set and $\rho_{\mathcal{X}}$ has a density with respect to Lebesgue measure, which is uniformly bounded away from 0 and $\infty$ on $\mathrm{supp}(\rho_\mathcal{X})$. Let $k$ be a Gaussian kernel and $\Sp$ be its RKHS; then, for any $m\geq d/2$, there exists a constant $C_{m, d}>0$ such that 
\begin{equation}
    \label{eqint1}
    \| f \|_{L^\infty(\rho_{\mathcal{X}})} \leq C_{m, d} \|f\|_{\Sp}^{\nicefrac{d}{2m}} \|f\|_{\Ltwo{\rho_\mathcal{X}}}^{1-\nicefrac{d}{2m}}
\end{equation}
for any $f \in \Sp$. Also, for any $M\geq 1$, let $k_M$ be a random Fourier features approximation of $k$ with $M$ features and $\Cbsp$ be a RKHS of $k+k_M$. Then, with probability at least $1-\delta$ with respect to a sampling of features, 
\begin{equation}
   \label{eqint2}
    \| f \|_{L^\infty(\rho_{\mathcal{X}})} \leq C_{m, d} \left(1 + \frac{1}{\delta} \right)^{\nicefrac{d}{4m}} \|f\|_{\Cbsp}^{\nicefrac{d}{2m}} \|f\|_{\Ltwo{\rho_\mathcal{X}}}^{1-\nicefrac{d}{2m}}
\end{equation}
for any $f \in \Cbsp$.
\end{thm}
We note that the norm relation of the Gaussian RKHS \eqref{eqint1} is a known result of \citet{steinwart2009optimal} and our analysis extends this to the combined RKHS $\Cbsp$.
The proof is based on the following fact: \\
Let us denote $\mathrm{supp}(\rho_{\mathcal{X}})$ by $\mathcal{X}^\prime$. First, from \citet{steinwart2009optimal} we have
\begin{align}
    \left[ L^2(\mathcal{X}^\prime), W^m(\mathcal{X}^\prime)\right]_{d/2m, 1} = B_{2, 1}^{d/2} (\mathcal{X}^\prime)
\end{align}
and there exists a constant $C_1>0$ such that
\begin{align}
    \| f \|_{\left[ L^2(\mathcal{X}^\prime), W^m(\mathcal{X}^\prime)\right]_{d/2m, 1}}
    \leq C_1 \|f\|_{W^m(\mathcal{X}^\prime)}^{\nicefrac{d}{2m}} \|f\|_{L^2(\mathcal{X}^\prime)}^{1-\nicefrac{d}{2m}},
\end{align}
where $W^m(\mathcal{X}^\prime)$ and $B_{2, 1}^{d/2} (\mathcal{X}^\prime)$ denote Sobolev and Besov space, respectively, and $[E, F]_{\theta, r}$ denotes real interpolation of Banach spaces $E$ and $F$ (see \citet{steinwart2008support}). 
Also, by Sobolev's embedding theorem for Besov space, $B_{2, 1}^{d/2} (\mathcal{X}^\prime)$ can be continuously embedded in $L^\infty(\mathcal{X}^\prime)$. Finally, from the condition on $\rho_{\mathcal{X}}$, there exists a constant $C_2>0$ such that
\begin{align}
    &\|f\|_{L^\infty(\rho_\mathcal{X})} = \|f\|_{L^\infty(\mathcal{X}^\prime)},\\
    &\|f\|_{\Ltwo{\rho_\mathcal{X}}} \geq C_2 \|f\|_{L^2({\mathcal{X}^\prime})}.
\end{align}
Therefore, if it can be shown that RKHS $\Cbsp$ is continuously embedded in $W^m(\mathcal{X}^\prime)$, the norm relation \eqref{eqint2} holds.
The complete proof is found in Appendix \ref{proof2}.
Using this theorem, it can be shown that in the case of a Gaussian kernel and its random Fourier features approximation, Assumption \ref{as2-2} is satisfied with $p=1/2$ and $C(\delta) = C_{d, d}(1+1/\delta)^{1/4}$, and the resulting rate in Theorem \ref{thm1} is $O(M^{-\nicefrac{1}{8}}).$

\section{Main Result}
\label{mainres}

In this section, we show that learning classification problems with SGD and random features achieve the exponential convergence of the expected classification error under certain conditions. Before providing our results, several assumptions are imposed on the classification problems and loss function. The first is the smoothness of the loss function.
\begin{asm}
\label{as2-8}
$l(\cdot, y)$ is differentiable and $L$-Lipschitz smooth. That is, for any $\zeta, \zeta^\prime \in \mathbb{R}$ and $y \in \mathcal{Y}$,
\begin{align}
    |\partial_{\zeta} l(\zeta, y) - \partial_{\zeta} l(\zeta^\prime, y)|   \leq L |\zeta - \zeta^\prime|.
\end{align}
\end{asm}
Let $l(g, z)$ denote $l(g(x), y)$ for $z=(x, y)$ and $\partial_g l(g ,z)$ denote the gradient of $l(g ,z)$ with respect to $g \in \Sp$. Combining Assumption \ref{as2-3} and \ref{as2-8} yields $LR^2$-smoothness in $\Sp$, since
\begin{align}
    &\langle \partial_g l(g,z) - \partial_g l(g^\prime,z) , g-g^\prime\rangle_\Sp \\&= \langle (\partial_\zeta l(g(x),y) - \partial_\zeta l(g^\prime(x),y))k(\cdot, x), g-g^\prime \rangle_\Sp \\
    &\leq LR^2 \|g-g^\prime\|^2_\Sp
\end{align}
holds for any $z \in \mathcal{X} \times \mathcal{Y}$ and it is known as an equivalent condition of smoothness by Theorem 2.1.5 of \citet{nesterov2014introductory}. The second is the margin condition on the conditional label probability.

\begin{asm}
\label{as2-5}
The strong low-noise condition holds:
    \begin{equation}
                \exists \delta \in \left(0, 1/2 \right), \ \  \left| \rho(Y=1|x) - 1/2 \right| > \delta \ \ (\rho_{\mathcal{X}}\text{-}a.s.)
    \end{equation}
\end{asm}

The third is the condition on the \textit{link function} $h_*$ \citep{bartlett2006convexity, zhang2004statistical}, which connects the hypothesis space and the probability measure:
\begin{equation}
     h_{*}(\mu) = \argmin_{\alpha \in \mathbb{R}} \left\{ \mu \phi(\alpha) + (1-\mu) \phi(-\alpha) \right\}.
\end{equation}
Its corresponding value is denoted by $l_*$:
\begin{equation}
        l_{*}(\mu) = \min_{\alpha \in \mathbb{R}} \left\{ \mu \phi(\alpha) + (1-\mu) \phi(-\alpha) \right\}.
\end{equation}
It is known that $l_*$ is a concave function \citep{zhang2004statistical}. Although $h_*(\mu)$ may not be uniquely determined nor well-defined in general, the following assumption ensures these properties.
\begin{asm}
\label{as2-6}
$\rho(1|X)$ takes values in $(0,1)$, $\rho_\mathcal{X}$-almost surely; $\phi$ is differentiable and $h_*$ is well-defined, differentiable, monotonically increasing, and invertible over $(0, 1)$.
Moreover, it follows that
\begin{equation}
    \sgn(\mu-1/2) = \sgn(h_*(\mu)).
\end{equation}
\end{asm}
For logistic loss, $h_*(\mu) = \log (\mu/(1-\mu))$, and the above condition is satisfied.
Next, following \citet{zhang2004statistical}, we introduce Bregman divergence for concave function $l_*$ to ensure the uniqueness of Bayes rule $g_*$:
\begin{equation}
    d_{l_*}(\eta_1, \eta_2) = -l_*(\eta_2) + l_*(\eta_1) + l_*^\prime (\eta_1) (\eta_2 - \eta_1).
\end{equation}
\begin{asm}
\label{as2-7}
Bregman divergence $d_{l_*}$ derived by $l_*$ is positive, that is, $d_{l_*}(\eta_1, \eta_2)=0$ if and only if $\eta_1=\eta_2$. For the expected risk $\mathcal{L}$, a unique Bayes rule $g_*$ (up to zero measure sets) exists in $\Sp$. 
\end{asm}
For logistic loss, it is known that $d_{l_*}$ coincides with Kullbuck-Leibler divergence, and thus, the positivity of the divergence holds. If $\phi$ is differentiable and $h_*$ is differentiable and invertible, the excess risk can be expressed using $d_{l_*}$ \citep{zhang2004statistical}:
\begin{equation}
    \mathcal{L}(g) - \mathcal{L}(g_*) = \mathbb{E}_X [d_{l_*}(h_*^{-1}(g(X)), \rho(1|X))]. 
\end{equation}
So, combining Assumptions \ref{as2-6} and  \ref{as2-7} implies that Bayes rule $g_*$ is equal to $h_*(\rho(1|X))$, $\rho_{\mathcal{X}}$-almost surely and contained in the original RKHS $\Sp$.
Finally, we introduce the following notation:
\begin{equation}
    m(\delta) = \max \{ h_*(0.5+\delta), |h_*(0.5-\delta)|\}.
\end{equation}
Using this notation, Assumption \ref{as2-5} can be reduced to the Bayes rule condition, that is, $|g_*(X)| \geq m(\delta)$, $\rho_{\mathcal{X}}$-almost surely. For logistic loss, we have $m(\delta)=\log((1+2\delta)/(1-2\delta))$.
Under these assumptions and notations, the exponential convergence of the expected classification error is shown.
\begin{thm}
\label{thm3}
Suppose Assumptions \ref{as2-1}--\ref{as2-7} hold. There exists a sufficiently small $\lambda>0$ such that the following statement holds:\\
Taking the number of random features $M$ that satisfies
\begin{equation}
        M \gtrsim \left(\frac{R^4 C^4(\delta^\prime)  G^3 \|g_*\|_{\mathcal{H}}^4}{\lambda^3 m^4(\delta)}\right)^{\max \left\{\frac{1}{1-p}, 2\right\}} R^4 \log \frac{R}{\delta^\prime}.
\end{equation}
Consider Algorithm \ref{alg1} with $\eta_t = \frac{2}{\lambda(\gamma+t)}$ and $\alpha_t = \frac{2(\gamma+t-1)}{(2\gamma+T)(T+1)}$ where $\gamma$ is a positive value such that $\|g_1\|_{\Rfsp} \leq (2\eta_1+1/\lambda)GR$ and $\eta_1 \leq \min \{1/LR^2, 1/2\lambda \}.$ Then, with probability $1-2\delta^\prime$, for sufficiently large $T$ such that
\begin{equation}
    \max \left\{ \frac{36G^2R^2}{\lambda^2(2\gamma+T)}, \frac{\gamma (\gamma-1) \|g_1 - g_{M, \lambda} \|_{\Rfsp}^2}{(2\gamma+T)(T+1)} \right \} \leq \frac{m^2(\delta)}{64R^2},
\end{equation}
we have the following inequality for any $t\geq T$:
\begin{align}
    \mathbb{E} \left[ \mathcal{R}(\overline{g}_{t+1}) - \mathcal{R}(\mathbb{E}[Y|x]) \right]
    \leq 2 \exp \left( -\frac{\lambda^2 (2\gamma + t) m^2(\delta)}{2^{12} \cdot 9 G^2R^4}  \right).
\end{align}
\end{thm}
The complete statement and proof are given in Appendix \ref{proof3}. We note that although a certain number of features are required to achieve the exponential convergence, the resulting rate does not depend on $M$. In contrast to this, when one considers the convergence rate of the loss function, its rate depends on $M$ in general \citep{rudi2017generalization, carratino2018learning, rahimi2009weighted, li2018toward}. From this fact, we can show that random features can save computational cost in a relatively small classification error regime. A detailed discussion is presented later.
\paragraph{Dependence of $\gamma$ and $\lambda$ on $T$}
As we can see in the condition inequality on $T$, $\gamma$ adjusts the step size and consequently affects $T$, when the exponential convergence phase starts. Indeed, there is a trade-off relation in $T$, that is, the first part of $\max$ in the condition on $T$ is $O(1/\gamma)$ and the second part is $O(\gamma)$.\\
In addition, we note that when we apply non-averaged SGD, the dependence of $\lambda$ on $T$ is worse than our averaged SGD, although similar exponential convergence can be shown in such case. This comes from the fact that gradient averaging achieves better dependence on the strongly convex parameter. For further details, see \citet{nitanda2018stochastic, lacoste2012simpler}.\\

As a corollary, we show a simplified result when learning with random Fourier features approximation of a Gaussian kernel and logistic loss, which can be obtained by setting $m(\delta)=\log((1+2\delta)/(1-2\delta))$, $R=G=1$ and $L=1/4$ in Theorem \ref{thm3} and applying Theorem \ref{thm2}. In addition, we can specify a required $\lambda$ to achieve the convergence in this setting. The complete statement and proof are given in Appendix \ref{proof4}.

\begin{cor}
Assume $\mathrm{supp}(\rho_\mathcal{X}) \subset \mathbb{R}^d$ is a bounded set and $\rho_{\mathcal{X}}$ has a density with respect to Lebesgue measure, which is uniformly bounded away from 0 and $\infty$ on $\mathrm{supp}(\rho_\mathcal{X})$. 
Let $k$ be a Gaussian kernel and $l$ be logistic loss. Under Assumption \ref{as2-5}-\ref{as2-7}, the following statement holds:\\
Taking a regularization parameter $\lambda$ and a number of random features $M$ that satisfies
\begin{align}
    &\lambda \lesssim \log^3 \frac{1+2\delta}{1-2\delta} \cdot \frac{1}{ (2+e^{R\|g_*\|_\Sp}+e^{-R\|g_*\|_\Sp})\|g_*\|_\Sp^3},\\
        &M \gtrsim \left(\frac{\left(1+\frac{1}{\delta^\prime}\right)  \|g_*\|_{\mathcal{H}}^4}{\lambda^3 \log^4 \frac{1+2\delta}{1-2\delta}}\right)^2 \log \frac{1}{\delta^\prime}.
\end{align}
Consider Algorithm \ref{alg1} with $\eta_t = \frac{2}{\lambda(\gamma+t)}$ and $\alpha_t = \frac{2(\gamma+t-1)}{(2\gamma+T)(T+1)}$ where $\gamma$ is a positive value such that $\|g_1\|_{\Rfsp} \leq (2\eta_1+1/\lambda)$ and $\eta_1 \leq \min \{4, 1/2\lambda \}.$ Then, with probability $1-2\delta^\prime$, for a sufficiently large $T$ such that
\begin{align}
\max \left\{ \frac{36}{\lambda^2(2\gamma+T)}, \frac{\gamma (\gamma-1) \|g_1 - g_{M, \lambda} \|_{\Rfsp}^2}{(2\gamma+T)(T+1)} \right \} &\\\leq \frac{\log^2 \frac{1+2\delta}{1-2\delta}}{64},&
\end{align}
we have the following inequality for any $t\geq T$:
\begin{align}
    &\mathbb{E} \left[ \mathcal{R}(\overline{g}_{t+1}) - \mathcal{R}(\mathbb{E}[Y|x]) \right] \\
   &\leq 2 \exp \left( -\frac{\lambda^2 (2\gamma + t)}{2^{12} \cdot 9} \log^2 \frac{1+2\delta}{1-2\delta} \right).
\end{align}
\end{cor}
\paragraph{Computational Viewpoint}
As shown in Theorem \ref{thm3}, once a sufficient number of features are sampled, the convergence rate of the excess classification error does not depend on the number of features $M$. This is unexpected because when considering the convergence of the loss function, the approximation error induced by random features usually remains \citep{rudi2017generalization, li2018toward, rahimi2009weighted}. Thus, to obtain the best convergence rate, we have to sample more $M$ as the sample size $n$ increases. \\
From this fact, it can be shown that to achieve a relatively small classification error, learning with random features is indeed more computationally efficient than learning with a full kernel method without approximation. As shown in Section \ref{comp}, if one runs SGD in Algorithm \ref{alg1} with more than $M$ iterations, both the time and space computational costs of a full kernel method exceed those of random features. In particular, if one can achieve a classification error $\epsilon$ such that
\begin{equation}
    \epsilon \lesssim \exp \left( - \log^{2\max\{\nicefrac{(1+p)}{(1-p)}, 3\}} m(\delta) \right),
\end{equation}
then the required number of iterations $n$ exceeds the required number of features $M$ in Theorem \ref{thm3}, and the overall computational cost become larger in a full kernel method. Theoretical results which suggest the efficiency of random features in terms of generalization error have only been derived in the regression setting \citep{rudi2017generalization, carratino2018learning}; this is the first time the superiority of random features has been demonstrated in the classification setting. Moreover, this result shows that an arbitrary small classification error is achievable as long as there is a sufficient number of random features unlike the regression setting where a required number of random features depend on the target accuracy.

\section{Experiments}
\label{exp}
In this section, the behavior of the SGD with random features studied on synthetic datasets is described. We considered logistic loss as a loss function, a Gaussian kernel as an original kernel function, and its random Fourier features as an approximation method. Two-dimensional synthetic datasets were used, as shown in Figure \ref{fig1}. The dataset support is composed of four parts: $[-1.0, -0.1] \times [-1.0, -0.1], [-1.0, -0.1] \times [0.1, 1.0],[0.1, 1.0] \times [-1, -0.1],[0.1, 1.0] \times [0.1, 1.0]$. For two of them, the conditional probability is $\rho(1|X) = 0.8$, and for the other two, $\rho(1|X) = 0.2$. This distribution satisfies the strong low-noise condition with $\delta = 0.3$. For hyper-parameters, we set $\gamma = 500$ and $\lambda = 0.001$. SGD was run 100 times with 12,000 iterations and the classification error and loss function were calculated on 100,000 test samples. The average of each run is reported with standard deviations.\\
\begin{figure}[t]
    \centering
    \includegraphics[clip, width=5cm]{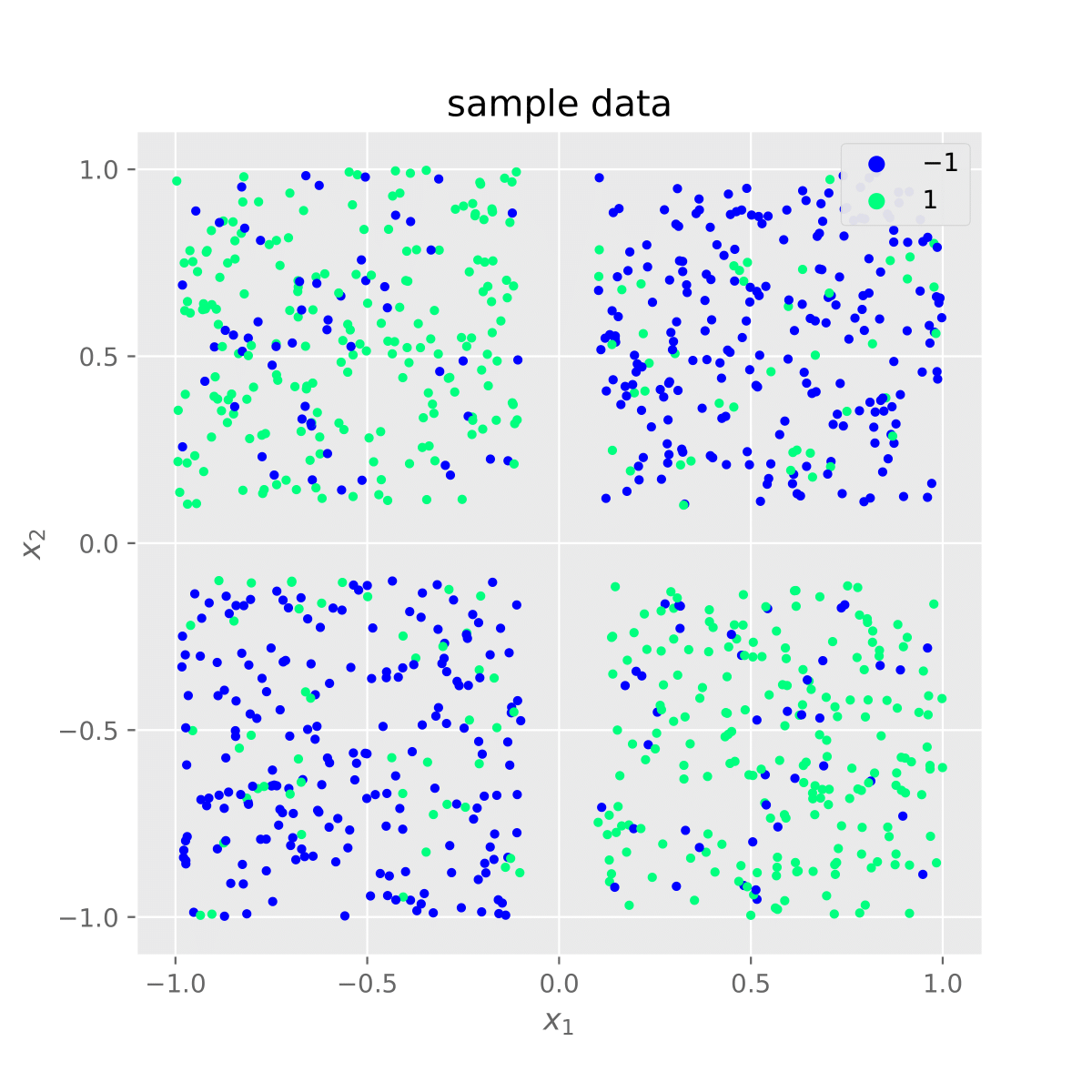}
    \caption{Subsample of data used in the experiment.}
    \label{fig1}
\end{figure}
\begin{figure}[t]
  \begin{minipage}[b]{0.48\linewidth}
    \centering
    \includegraphics[clip, width=\linewidth]{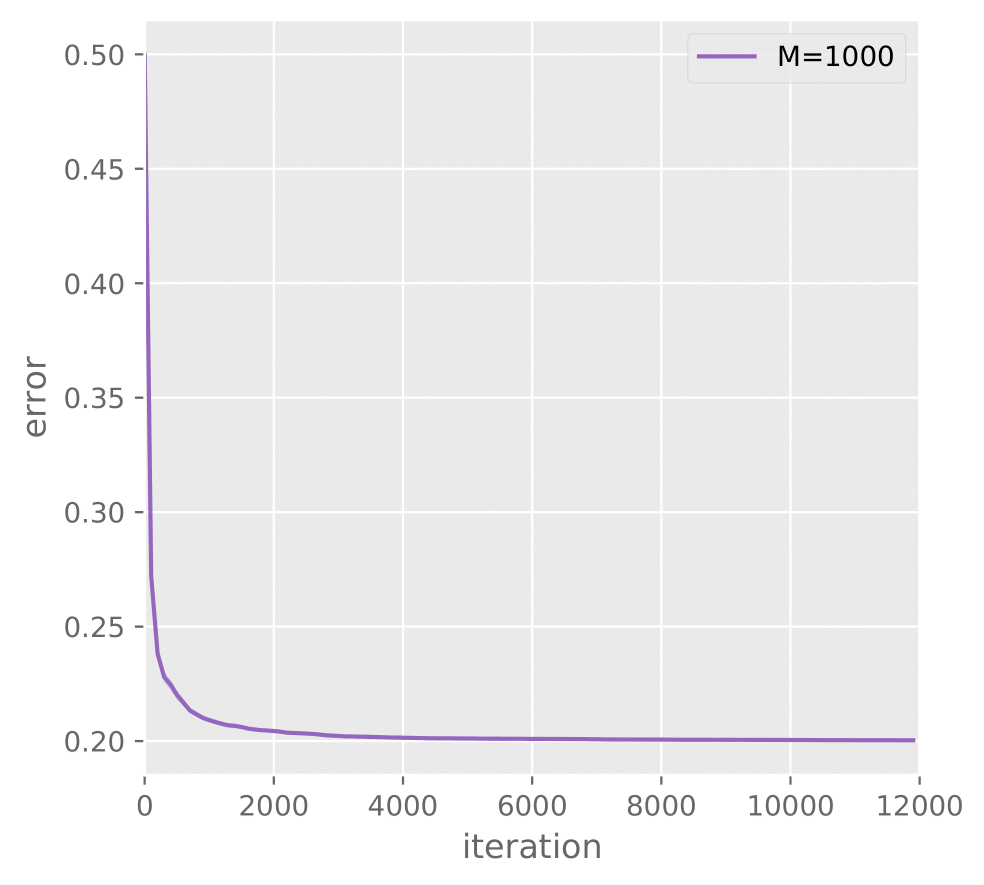}
    \subcaption{Classification errors}
  \end{minipage}
  \begin{minipage}[b]{0.48\linewidth}
    \centering
    \includegraphics[clip, width=\linewidth]{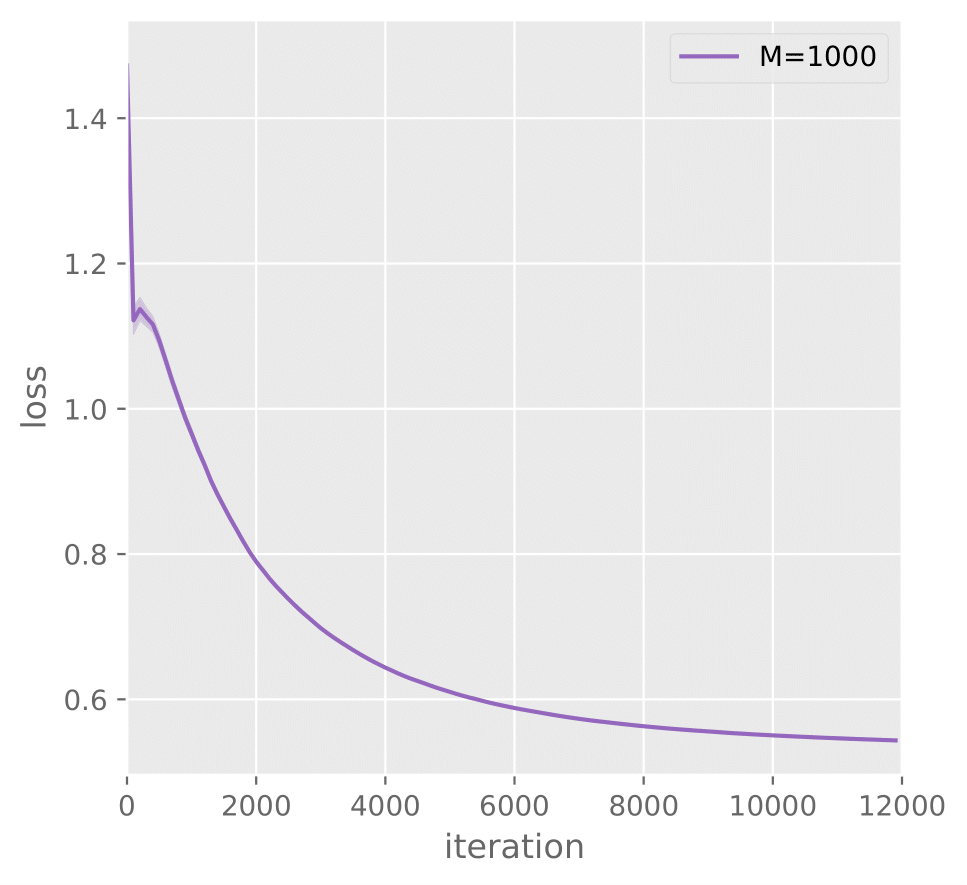}
    \subcaption{Loss functions}
\end{minipage}
\caption{Learning curves of the expected classification error (left) and the expected loss function (right) by averaged SGD with 1000 features. }
\label{fig2}
\end{figure}
First, the learning curves of the expected classification error and the expected loss function are drawn when the number of features $M = 1000$, as shown in Figure \ref{fig2}. Our theoretical result suggests that with sufficient features, the classification error converges exponentially fast, whereas the loss function converges sub-linearly. We can indeed observe a much faster decrease in the classification error (left) than in the loss function (right). Next, we show the learning curves of the expected classification error when the number of features are varied as $M = 100, 200, 500, 1000$ in Figure \ref{fig4}. We can see that the exact convergence of the classification error is not attained with relatively few features such as $M = 100$, which also coincides with our results. 
\begin{figure}
    \centering
    \includegraphics[clip, width=6cm]{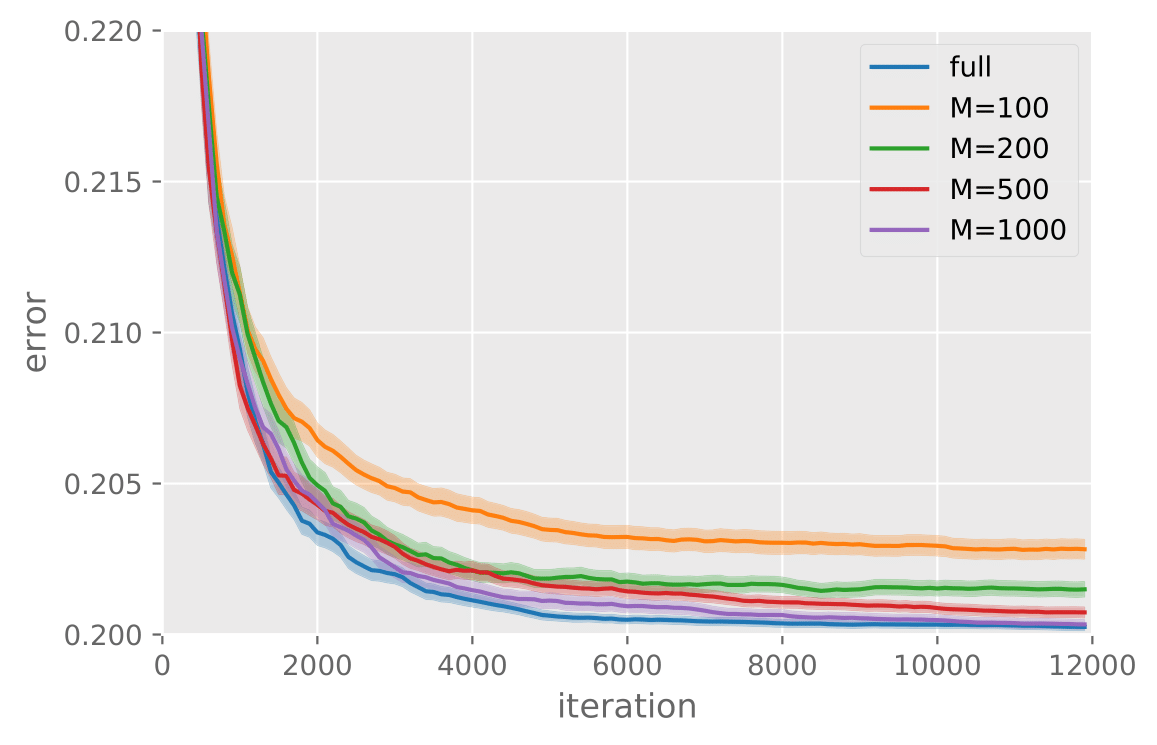}
    \caption{Comparison of learning curves of the expected classification error with varying numbers of features.}
    \label{fig4}
\end{figure}
Finally, the convergence of the classification error is compared in terms of computational cost between the random features model with $M = 500, 1000$ and the full kernel model without approximation. In Figure \ref{fig5}, the learning curves are drawn with respect to the number of parameter updates; the full kernel model requires increasing numbers of updates in later iterations, whereas the random features model requires a constant number of updates. It can be observed that both random features models require fewer parameter updates to achieve the same classification error than the full kernel model for a relatively small classification error. This implies that random features approximation is indeed computationally efficient under a strong low-noise condition.

\begin{figure}
    \centering
    \includegraphics[clip, width=6cm]{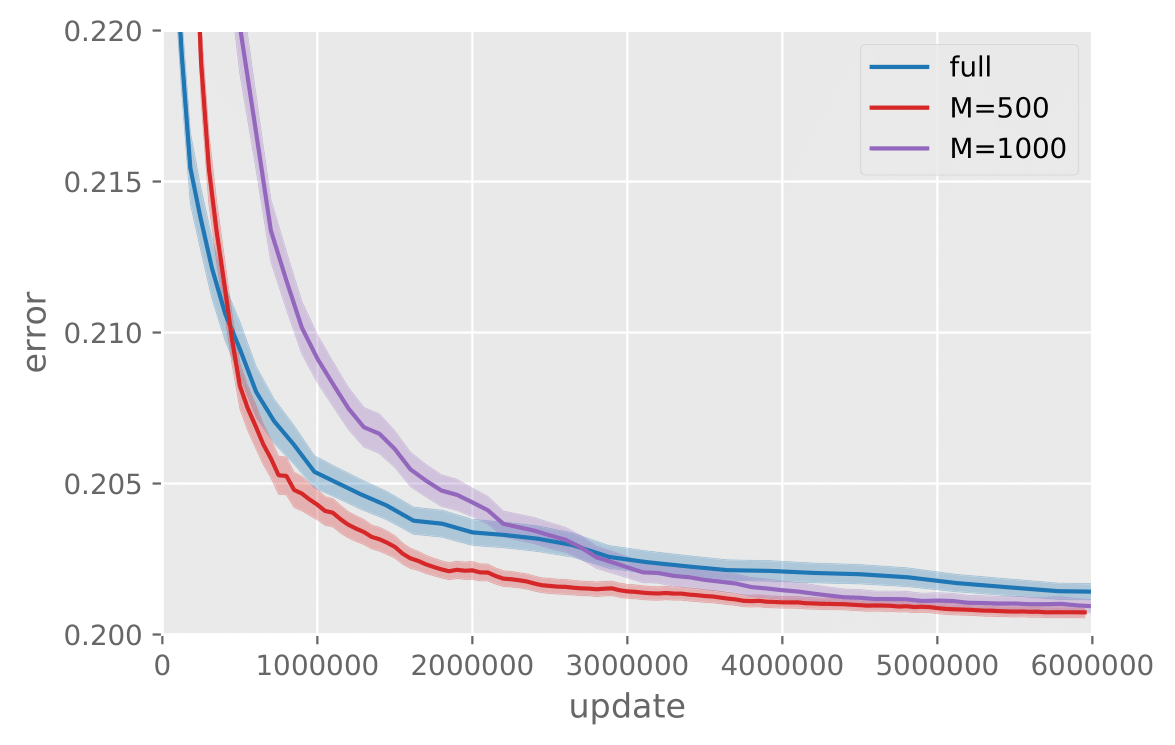}
    \caption{Comparison of learning curves with respect to number of parameter updates.}
    \label{fig5}
\end{figure}

\section{Conclusion}
This study shows that learning with SGD and random features could achieve exponential convergence of the classification error under a strong low-noise condition. Unlike when considering the convergence of a loss function, the resulting convergence rate of the classification error is independent of the number of features, indicating that an arbitrary small classification error is achievable as long as there is a sufficient number of random features. Our results suggest, for the first time, that random features is theoretically computationally efficient even for classification problems under certain settings. Our theoretical analysis has been verified by numerical experiments.
\\
One possible future direction is to extend our analysis to general low-noise conditions to derive faster rates than $O(1/\sqrt{n})$, as \citet{pillaud2018exponential} did in the case of the squared loss. It could also be interesting to explore the convergence speed of more sophisticated variants of SGD, such as stochastic accelerated methods and stochastic variance reduced methods \citep{schmidt2017minimizing, johnson2013accelerating, defazio2014saga, allen2017katyusha}.

\subsubsection*{Acknowledgements}
AN was partially supported by JSPS KAKENHI (19K20337) and JST PRESTO. TS was partially supported by JSPS KAKENHI (18K19793, 18H03201, and 20H00576), Japan Digital Design, and JST CREST.






\bibliography{ref}
\bibliographystyle{apalike}
\onecolumn
\renewcommand{\thesection}{\Alph{section}}
\setcounter{section}{0}
\setcounter{thm}{0}
\setcounter{cor}{0}

\part*{Appendix}

\section{Notation and Useful Propositions}
\label{notation}
Let $V$ be a Hilbert space with inner product $\langle \cdot, \cdot\rangle_V$ and the induced norm $\|\cdot\|_V$. For $A : V\rightarrow V$, we denote an operator norm of $A$ as $\|A\|_{\mathrm{op}}$, that is,
\begin{align}
    \|A\|_{\mathrm{op}} \defeq \sup_{v \in V} \frac{\|Av\|_{V}}{\|v\|_{V}}.
\end{align}
For $a, b \in V$, we define an outer product $a \otimes_V b:V\rightarrow V$ as follows:
\begin{align}
    (a \otimes_V b)v \defeq \langle b, v \rangle_V a, \ \ \forall v \in V.
\end{align}
Let $W$ be a closed subspace of $V$, then a projection onto $W$ is well defined and we denote its operator by $\mathcal{P}_{W}$. Then we have
\begin{align}
    v = \mathcal{P}_{W} v + \mathcal{P}_{W^\bot} v, \ \ \forall v \in V.
\end{align}
Furthermore, we define a partial order $\preceq$ between linear, positive semi-definite and self-adjoint operators $A, B: V \rightarrow V$ as follows:
\begin{align}
    A \preceq B \quad\overset{\mathrm{def}}{\Longleftrightarrow} \quad \langle Av, v\rangle_V \leq \langle Bv, v\rangle_V, \ \forall v \in V.
\end{align}

The following inequality shows that the difference between the square root of two self-adjoint positive semi-definite operators is bounded by the square root of the difference of them.
\begin{prop}
\label{b}
Let $V$ be a separable Hilbert space. For any compact, positive semi-definite, self-adjoint operators $S, \widetilde{S}:V \rightarrow V$, the following inequality holds:
\begin{equation}
    \| S^{\nicefrac{1}{2}} - \widetilde{S}^{\nicefrac{1}{2}} \|_{\mathrm{op}} \leq 
    \| S - \widetilde{S} \|_{\mathrm{op}}^{\nicefrac{1}{2}}
    \label{eqa4}
\end{equation}
\end{prop}
\begin{proof}
Since $S^{\nicefrac{1}{2}} - \widetilde{S}^{\nicefrac{1}{2}}$ is also a compact and self-adjoint operator, it allows eigendecomposition of itself. Then let $\lambda_{\max}$ be the eigenvalue with largest absolute value and $v$ be the corresponding normalized eigenfunction of 
$S^{\nicefrac{1}{2}} - \widetilde{S}^{\nicefrac{1}{2}}$, i.e.,
\begin{equation}
    (S^{\nicefrac{1}{2}} - \widetilde{S}^{\nicefrac{1}{2}}) v = \lambda_{\max} v.
\end{equation}
Since \eqref{eqa4} obviously holds if $S=\widetilde{S}$, we can assume that $\lambda_{\max} > 0$ without loss of generality. Because $S^{\nicefrac{1}{2}}$ is positive semi-definite, we have
\begin{align}
    \langle v, S v \rangle_{V} &= \|S^{\nicefrac{1}{2}} v\|_{V}^2 \\
    &= \|\widetilde{S}^{\nicefrac{1}{2}} v + \lambda_{\max} v\|_{V}^2 \\
    &= \langle v, \widetilde{S} v \rangle_{V}
    + \lambda_{\max}^2 + 2\lambda_{\max} \langle v, S^{\nicefrac{1}{2}} v \rangle_{V} \\
    &\geq \langle v, \widetilde{S} v \rangle_{V} + \lambda_{\max}^2.
\end{align}
Thus we have
\begin{align}
    \| S - \widetilde{S} \|_{\mathrm{op}} &\geq  \langle v, (S - \widetilde{S}) v \rangle_{V} \\
    &\geq \lambda_{\max}^2 = \| S^{\nicefrac{1}{2}} - \widetilde{S}^{\nicefrac{1}{2}} \|_{\mathrm{op}}^2,
\end{align}
which completes the proof.
\end{proof}
The following inequality is a generalization of the Bernstein inequality to random operators on separable Hilbert space and used in Lemma \ref{lem3} to derive the concentration of integral operators.
\begin{prop}[Proposition 3 in \cite{rudi2017generalization}]
\label{pro1}
Let $V$ be a separable Hilbert space and let $X_1, X_2, \ldots , X_n$ be a sequence of independent and identically distributed self-adjoint random operators on $V$. Assume that $\mathbb{E}X_i = 0$ and there exists $B > 0$ such that $\|X_i\|_{\mathrm{op}} \leq B$ almost surely for any $1\leq i\leq n$. Let S be the positive operator such that $\mathbb{E}X_i^2 \leq S$. Then for any $\delta \in (0,1]$, the following inequality holds with probability at least $1-\delta$:
\begin{equation}
    \left\|\frac{1}{n} \sum_{i=1}^n X_i \right\|_{\mathrm{op}} \leq \frac{2B \beta}{3n} + \sqrt{\frac{2 \|S\|_{\mathrm{op}} \beta}{n}},
\end{equation}
where $\beta = \log \frac{2 \mathrm{tr} S}{\|S\|_{\mathrm{op}} \delta}.$
\end{prop}

\section{Basic Properties of RKHS}
In analyses of kernel methods, it is common to assume $\mathcal{X}$ is compact, $\rho_{\mathcal{X}}$ has the full support and $k$ is continuous because under such assumptions we utilize Mercer's theorem to characterize RKHS \cite{cucker2002mathematical, aronszajn1950theory}. However, such an assumption may not be adopted under the strong low noise condition in which $\rho_{\mathcal{X}}$ may not have full support. In this section, we explain some basic properties of reproducing kernel Hilbert space (RKHS) under more general settings based on \cite{dieuleveut2016nonparametric, steinwart2012mercer}.\\
First, for given kernel function $k$ and its RKHS $\Sp$, we define a covariance operator $\Sigma:\mathcal{H} \rightarrow \mathcal{H}$ as follows:
\begin{align}
    \langle f, \Sigma g \rangle_{\mathcal{H}} = \langle f,  g \rangle_{\Ltwo{\rho_\mathcal{X}}}, \ \ \forall f, g \in \Sp.
\end{align}
It is well-defined through Riesz' representation theorem. Using reproducing property, we have
\begin{align}
    &\Sigma = \mathbb{E}_{X \sim \rho_{\mathcal{X}}}[k(\cdot, X) \otimes_{\Sp} k(\cdot, X)], \\
    &(\Sigma f)(z) = \mathbb{E}_{X \sim \rho_{\mathcal{X}}}[f(X)k(X, z)],\ \ \forall f \in \Sp. \label{intop1}
\end{align}
where expectation is defined via a Bochner integration. From the representation \eqref{intop1}, we can extend the covariance operate to $f \in \Ltwo{\rho_{\mathcal{X}}}$. We denote this by $T:\Ltwo{\rho_{\mathcal{X}}} \rightarrow \Ltwo{\rho_{\mathcal{X}}}$ as follows:
\begin{align}
    (T f)(z) =  \mathbb{E}_{X \sim \rho_{\mathcal{X}}}[f(X)k(X, z)],\ \ \forall f \in \Ltwo{\rho_{\mathcal{X}}}. \label{intop2}
\end{align}
$\mathrm{Im}(T) \subset \Ltwo{\rho_{\mathcal{X}}}$ is verified since $k(\cdot, x)$ is uniformly bounded by Assumption \ref{as2-3}.
Also, we can write $T$ using feature expansion \eqref{eq1} as
\begin{align}
    T = \mathbb{E}_{\omega \sim \tau}[\varphi(\cdot, \omega) \otimes_{\Ltwo{\rho_\mathcal{X}}} \varphi(\cdot, \omega)], \label{intop4}
\end{align}
since \begin{align}
    (T f)(z) &= \mathbb{E}_{X \sim \rho_{\mathcal{X}}}[f(X) \mathbb{E}_{\omega \sim \tau}[\varphi(X, \omega) \varphi(z, \omega)]] \\
    &= \mathbb{E}_{\omega \sim \tau} [\langle f, \varphi(\cdot, \omega) \rangle_{\Ltwo{\rho_\mathcal{X}}}\varphi(z, \omega)].
\end{align}
Following \cite{dieuleveut2016nonparametric}, here we denote a set of square integral function itself by $\mathcal{L}^2(d\rho_{\mathcal{X}})$, that is, its quotient is $\Ltwo{\rho(\mathcal{X})}$, which is separable Hilbert space. We can also define the extended covariance operator  $\mathcal{T}:\Ltwo{\rho_{\mathcal{X}}} \rightarrow \mathcal{L}^2(d\rho_{\mathcal{X}})$ as follows:
 \begin{align}
    (\mathcal{T} f)(z) =  \mathbb{E}_{X \sim \rho_{\mathcal{X}}}[f(X)k(X, z)],\ \ \forall f \in \Ltwo{\rho_{\mathcal{X}}}. \label{intop3}
\end{align}
Here we present some properties of these covariance operators $\Sigma, T, \mathcal{T}$ from \cite{dieuleveut2016nonparametric}.

\begin{prop} \mbox{} \label{prop1}
\begin{enumerate} 
    \item $\Sigma$ is self-adjoint, continuous operator and $\mathrm{Ker}(\Sigma) = \{ f \in \Sp \ |\  \|f\|_{\Ltwo{\rho_{\mathcal{X}}}} = 0 \}.$
    \item $T$ is continuous, self-adjoint, positive semi-definite operator.
    \item $\mathcal{T}^{\nicefrac{1}{2}}: \mathrm{Ker}(T)^\bot \rightarrow \mathrm{Ker}(\Sigma)^\bot$ is well-defined and an isometry. 
    In particular, for any $f \in \mathrm{Ker}(\Sigma)^\bot \subset \Sp$, there exists $g \in \mathrm{Ker}(T)^\bot \subset \Ltwo{\rho_{\mathcal{X}}}$ such that $\|f\|_\Sp = \|g \|_{\Ltwo{\rho_{\mathcal{X}}}}.$
\end{enumerate}
\end{prop}
We denote the extended covariate operator associate with $k_M$ by $T_M:\Ltwo{\rho_{\mathcal{X}}} \rightarrow \Ltwo{\rho_{\mathcal{X}}}$ and $\mathcal{T}_M:\Ltwo{\rho_{\mathcal{X}}} \rightarrow \mathcal{L}^2(d\rho_{\mathcal{X}})$. 

As with \eqref{intop4}, we have
\begin{align}
     &T_M = \frac{1}{M} \sum_{i=1}^M \varphi(\cdot, \omega_i) \otimes_{\Ltwo{\rho_\mathcal{X}}} \varphi(\cdot, \omega_i), \label{intop6} \\
     &\mathbb{E}[T_M] = T.
\end{align}
The next lemma provides a probabilistic bounds about the difference of the two covariate operators $T$ and $T_M$.

\begin{lem}
\label{lem3}
For any $\delta \in [0, 1),$ the following inequality holds with probability at least $1-\delta$:
\begin{equation}
    \left \| T - T_M \right \|_{\mathrm{op}} \leq R^2\left(\frac{2\beta}{3M} + \sqrt{\frac{2\beta}{M}}\right)
\end{equation}
where $\beta = \log \frac{2 R^2}{\| T \|_{\mathrm{op}} \delta}.$
\end{lem}

\begin{proof}
Let $X_i = T - \varphi(\cdot, \omega_i) \otimes_{\Ltwo{\rho_\mathcal{X}}} \varphi(\cdot, \omega_i).$ Then $T - T_M = \frac{1}{M} \sum_{i=1}^M X_i.$ Also, we have
\begin{align}
    \mathbb{E} X_i &= 0, \\
    X_i &\preceq T \preceq R^2 I, \\
    X_i &\succeq - \varphi(\cdot, \omega_i) \otimes_{\Ltwo{\rho_\mathcal{X}}} \varphi(\cdot, \omega_i) \succeq -R^2 I, \\
    ||X_i||_{\mathrm{op}} &\leq R^2, \text{as a result of two previous inequalities,} \\
    \mathbb{E}X_i^2 &= \mathbb{E}\left[ \varphi(\cdot, \omega_i) \otimes_{\Ltwo{\rho_\mathcal{X}}} \varphi(\cdot, \omega_i) \right]^2 - T ^ 2\\
    &\preceq \mathbb{E}\left[ \varphi(\cdot, \omega_i) \otimes_{\Ltwo{\rho_\mathcal{X}}} \varphi(\cdot, \omega_i) \right]^2 \\
    &\preceq \mathbb{E} \left[ \langle \varphi(\cdot, \omega_i), \varphi(\cdot, \omega_i) \rangle_{\Ltwo{\rho_\mathcal{X}}} \varphi(\cdot, \omega_i) \otimes_{\Ltwo{\rho_\mathcal{X}}} \varphi(\cdot, \omega_i) \right] \\
    &\preceq R^2 T, \\
    \mathrm{tr}T &= \int_\mathcal{X} k(x, x) \mathrm{d}\rho_\mathcal{X}(x) \leq R^2.
\end{align}
Let $B= R^2$ and $S = R^2 T$ in Proposition \ref{pro1}, we have
\begin{align}
    \left \| T - T_M \right \|_{\mathrm{op}} &= \left \| \frac{1}{M} \sum_{i=1}^M X_i \right \|_{\mathrm{op}} \\
    &\leq \frac{2R^2 \beta}{3M} + \sqrt{\frac{2R^2 \| T \|_{\mathrm{op}} \beta}{M}} \\
    &\leq  R^2\left(\frac{2\beta}{3M} + \sqrt{\frac{2\beta}{M}}\right),
\end{align}
which completes the proof.
\end{proof}
Let $\Sp$ and $\Rfsp$ be RKHSs associate with kernels $k$ and $k_M$, respectively.
Using Proposition \ref{prop1} and Lemma \ref{lem3}, we have the following proposition, which is essential in the proof of Theorem \ref{thm1}.
\begin{lem}
\label{lem1}
For any $\delta \in (0, 1]$ and $\xi > 0$, if
$$M \geq \max \left\{ \frac{8}{3} \left( \frac{R}{\xi} \right)^2, 32 \left( \frac{R}{\xi} \right)^4 \right\} \log \frac{2R^2}{\| T \|_{\mathrm{op}} \delta}$$
holds, the following statement holds with probability at least $1-\delta$:\\
For any $g \in \Sp,$ there exists $\widetilde{g} \in \Rfsp$ that satisfies
\begin{itemize}
    \item $\| g - \widetilde{g} \|_{\Ltwo{\rho_\mathcal{X}}} \leq \xi \|g\|_{\Sp}$
    \item $\| g \|_{\Sp} \geq \| \widetilde{g} \|_{\Rfsp}$.
\end{itemize}
Also, for any $\widetilde{g} \in \Rfsp,$ there exists $g \in \Sp$ that satisfies
\begin{itemize}
    \item $\| g - \widetilde{g} \|_{\Ltwo{\rho_\mathcal{X}}} \leq \xi \|\widetilde{g}\|_{\Rfsp}$
    \item $\| g \|_{\Sp} \leq \| \widetilde{g} \|_{\Rfsp}$.
\end{itemize}
\end{lem}

\begin{proof}
We show the first part of the statement. The latter half can be shown in the same manner. \\For $g \in \Sp,$ set $\widetilde{g} = \mathcal{T}_M^{\nicefrac{1}{2}} \mathcal{P}_{\mathrm{Ker}(T_M)^\bot} \mathcal{T}^{-\nicefrac{1}{2}} \mathcal{P}_{\mathrm{Ker}(\Sigma)^\bot}g \in \Rfsp.$ Then we have
\begin{align}
    \|\widetilde{g}\|_{\Rfsp} &= \| \mathcal{P}_{\mathrm{Ker}(T_M)^\bot} \mathcal{T}^{-\nicefrac{1}{2}} \mathcal{P}_{\mathrm{Ker}(\Sigma)^\bot}g \|_{\Ltwo{\rho_{\mathcal{X}}}} \\
    &\leq  \| \mathcal{T}^{-\nicefrac{1}{2}} \mathcal{P}_{\mathrm{Ker}(\Sigma)^\bot}g \|_{\Ltwo{\rho_{\mathcal{X}}}} \\
    &= \| \mathcal{P}_{\mathrm{Ker}(\Sigma)^\bot}g \|_{\Sp} \\
    &\leq \| g \|_{\Sp}.
\end{align}

Moreover, by Proposition \ref{b} and Lemma \ref{lem3}, with probability at least $1-\delta$, we have
\begin{align}
    \| g - \widetilde{g} \|_{\Ltwo{\rho_\mathcal{X}}} &= \| \mathcal{P}_{\mathrm{Ker}(\Sigma)^\bot}g - \widetilde{g} \|_{\Ltwo{\rho_\mathcal{X}}} \ \ \ (\because \text{Proposition \ref{prop1}.1})\\
    &= \| \mathcal{T}^{\nicefrac{1}{2}} h - \mathcal{T}_M^{\nicefrac{1}{2}} \mathcal{P}_{\mathrm{Ker}(T_M)^\bot} h \|_{\Ltwo{\rho_\mathcal{X}}}\\
    &= \| T^{\nicefrac{1}{2}} h - T_M^{\nicefrac{1}{2}} h \|_{\Ltwo{\rho_\mathcal{X}}}\\
    &\leq \|T^{\nicefrac{1}{2}} - T_M^{\nicefrac{1}{2}} \|_{\mathrm{op}} \|h\|_{\Ltwo{\rho_\mathcal{X}}} \\
    &\leq \|T - T_M \|^{\nicefrac{1}{2}}_{\mathrm{op}} \|g\|_{\Sp} \\
    &\leq \left(R^2\left(\frac{2\beta}{3M} + \sqrt{\frac{2\beta}{M}}\right)\right)^{\nicefrac{1}{2}} \|g\|_{\Sp} \\
    &\leq R\left(\left( \frac{2\beta}{3M}\right)^{\nicefrac{1}{2}} +  \left(\frac{2\beta}{M}\right)^{\nicefrac{1}{4}}\right) \|g\|_{\Sp}
\end{align}
where $h = \mathcal{T}^{-\nicefrac{1}{2}} \mathcal{P}_{\mathrm{Ker}(\Sigma)^\bot}g \in \Ltwo{\rho_{\mathcal{X}}}$ and $\beta = \log \frac{2 R^2}{\| T \|_{\mathrm{op}} \delta}.$\\
Solving the equation $\max \left\{ \left( \frac{2\beta}{3M}\right)^{\nicefrac{1}{2}}, \left(\frac{2\beta}{M}\right)^{\nicefrac{1}{4}}\right\}  \leq \frac{\xi}{2R}$, 
we get a desired result.
\end{proof}

\section{Proof of Theorem 1}
In this section, we give the complete statement and proof of Theorem \ref{thm1}.
\label{proof1}

\begin{thm}
Define $\xi > 0$ such that
\begin{align}
    \xi = \min \left\{ \left(\frac{\epsilon}{2^{p+1} C(\delta) \|g_*\|_\Sp} \right)^{\nicefrac{1}{1-p}}, \frac{\lambda \epsilon^2}{2^4 \cdot 3 R^2L\|g_*\|_\Sp},  \left( \frac{\lambda^3 \epsilon^4}{2^7 \cdot 3^2 R^4 L^2 \mathcal{L}(g_*)} \right)^{\nicefrac{1}{2}}, \left( \frac{\lambda^3 \epsilon^4}{2^7 \cdot 3^2 R^4 L^3 \|g_*\|_\Sp} \right)^{\nicefrac{1}{3}} \right\}.
\end{align}
Then a number of random features $M$ which satisfies
\begin{equation}
    M \geq \max \left\{ \frac{8}{3} \left( \frac{R}{\xi} \right)^2, 32 \left( \frac{R}{\xi} \right)^4 \right\} \log \frac{2R^2}{\| T \|_{\mathrm{op}} \delta}
\end{equation}
is enough to guarantee, with probability at least $1-2\delta$, that
\begin{equation}
    \|\G - \GG\|_{L^\infty(\rho_{\mathcal{X}})} \leq \epsilon.
\end{equation}
\end{thm}

\begin{proof}
By Lemma \ref{lem1}, for given $\xi > 0$, if we have a number of feature $M$ such that
\begin{equation}
        M \geq \max \left\{ \frac{8}{3} \left( \frac{R}{\xi} \right)^2, 32 \left( \frac{R}{\xi} \right)^4 \right\} \log \frac{2R^2}{\| T \|_{\mathrm{op}} \delta},
\end{equation}
we can take $\GGG \in \Rfsp, \GGGG \in \Sp$ which satisfy the following conditions:
\begin{align}
&\|\G\|_{\Sp} \geq \|\GGG\|_{\Rfsp} \label{a1}\\ 
&\|\GG\|_{\Rfsp} \geq \|\GGGG\|_{\Sp} \label{a2}\\
&\|\GGGG - \GG\|_{\Ltwo{\rho_\mathcal{X}}} \leq \xi \|\GG\|_{\Rfsp} \label{a3}\\
&\|\GGG - \G\|_{\Ltwo{\rho_\mathcal{X}}} \leq \xi \|\G\|_{\Sp} \label{a4}
\end{align}
By $\lambda$-strong convexity with respect to RKHS norm, we have
\begin{align}
    \loss(\G) + \frac{\lambda}{2} \|\G\|_{\Sp}^2 + \frac{\lambda}{2} \|\G-\GGGG\|_{\Sp}^2 
    \leq \loss(\GGGG) + \frac{\lambda}{2} \|\GGGG\|_{\Sp}^2 \label{1}\\
     \loss(\GG) + \frac{\lambda}{2} \|\GG\|_{\Rfsp}^2 + \frac{\lambda}{2} \|\GG-\GGG\|_{\Rfsp}^2 
    \leq \loss(\GGG) + \frac{\lambda}{2} \|\GGG\|_{\Rfsp}^2. \label{2}
\end{align}
In addition, by $L$-Lipschitzness of $\mathcal{L}$ with respect to $\Ltwo{\rho_{\mathcal{X}}}$ norm in Assumption \ref{as2-1} and \eqref{a3}\eqref{a4}, we have
\begin{align}
    \loss(\GGGG) &\leq \loss(\GG) + L \|\GGGG - \GG\|_{\Ltwo{\rho_\mathcal{X}}}\\
    &\leq \loss(\GG) +  L\xi  \|\GG\|_{\Rfsp} \label{3} \\
    \loss(\GGG) &\leq \loss(\G) + L \|\GGG - \G\|_{\Ltwo{\rho_\mathcal{X}}}\\
    &\leq \loss(\G) +  L\xi  \|\G\|_{\Sp} \label{4}
\end{align}
By inequalities \eqref{1}\eqref{2}\eqref{3}\eqref{4} and \eqref{a1}\eqref{a2}, we have
\begin{align}
    &\loss(\G) + \frac{\lambda}{2} \|\G\|_{\Sp}^2 + \frac{\lambda}{2} \left(\|\G - \GGGG\|_{\Sp}^2 +  \|\GG-\GGG\|_{\Rfsp}^2\right) \\
    &\leq \loss(\GGGG) + \frac{\lambda}{2} \|\GGGG\|_{\Sp}^2 + \frac{\lambda}{2} \|\GG-\GGG\|_{\Rfsp}^2 \\
    &\leq \loss(\GG) + L\xi \|\GG\|_{\Rfsp}+ \frac{\lambda}{2} \|\GG\|_{\Rfsp}^2 + \frac{\lambda}{2} \|\GG-\GGG\|_{\Rfsp}^2  \\
    &\leq \loss(\GGG) + \frac{\lambda}{2} \|\GGG\|_{\Rfsp}^2 +  L\xi \|\GG\|_{\Rfsp}\\
    &\leq \loss(\G)  + \frac{\lambda}{2} \|\G\|_{\Sp}^2 +  L\xi \left(\|\G\|_{\Sp} + \|\GG\|_{\Rfsp}\right).
\end{align}
Thus we have
\begin{equation}
    \|\G - \GGGG\|_{\Sp}^2 +  \|\GG-\GGG\|_{\Rfsp}^2 \leq \frac{2L\xi}{\lambda}  \left(\|\G\|_{\Sp} + \|\GG\|_{\Rfsp}\right). \label{c1}
\end{equation}


In addition, by \eqref{2} and \eqref{4}, we have
\begin{align}
\frac{\lambda}{2} \|\GG\|_{\Rfsp}^2 &\leq \loss(\GGG) + \frac{\lambda}{2} \|\GGG\|_{\Rfsp}^2 \\
&\leq \loss(\G) + L\xi \|\G\|_{\Sp} +\frac{\lambda}{2} \|\G\|_{\Sp}^2 \label{c2}.
\end{align}

Combining \eqref{c1} and \eqref{c2}, we obtain
\begin{align}
   \|\G - \GGGG\|_{\Sp}^2 +  \|\GG-\GGG\|_{\Rfsp}^2 &\leq  \frac{2L\xi}{\lambda}\left( \| \G \|_\Sp + \left(\frac{2}{\lambda} \mathcal{L}(\G) + \frac{2L\xi}{\lambda} \|\G\|_\Sp+ \|\G\|_\Sp^2 \right)^{\nicefrac{1}{2}} \right)\\
   &\leq  \frac{2L\xi}{\lambda}\left( \| g_* \|_\Sp + \left(\frac{2}{\lambda} \mathcal{L}(g_*) + \frac{2L\xi}{\lambda} \|g_*\|_\Sp+ \|g_*\|_\Sp^2 \right)^{\nicefrac{1}{2}} \right)\\
   &\leq \frac{2L\xi}{\lambda} \left(2\|g_*\|_\Sp + \left(\frac{2}{\lambda} \mathcal{L}(g_*) \right)^{\nicefrac{1}{2}} + \left( \frac{2L\xi}{\lambda} \|g_*\|_\Sp \right)^{\nicefrac{1}{2}} \right).
\end{align}
In the second inequality, we used $\|g_*\|_\Sp \geq \|\G\|_\Sp$ and $\mathcal{L}(g_*)+\frac{\lambda}{2} \|g_*\|_\Sp^2 \geq \mathcal{L}(\G)+\frac{\lambda}{2} \|\G\|_\Sp^2$. In the third inequality, we used $\sqrt{a} + \sqrt{b} \geq \sqrt{a+b}$ for $a, b>0$. Then by Assumption \ref{as2-3}, we obtain
\begin{align}
    \label{c4}
    \|\GG - \GGG\|_{L^\infty(\rho_{\mathcal{X}})} \leq R\max \left\{ \left(\frac{12L\xi}{\lambda} \|g_*\|_\Sp \right)^{\nicefrac{1}{2}}, \left(\frac{72L^2\xi^2}{\lambda^3} \mathcal{L}(g_*) \right)^{\nicefrac{1}{4}}, \left( \frac{72L^3\xi^3}{\lambda^3} \|g_*\|_\Sp\right)^{\nicefrac{1}{4}} \right\}.
\end{align}
On the other hand, by Assumption \ref{as2-2}, we have
\begin{align}
    \|\G - \GGG\|_{L^\infty(\rho_{\mathcal{X}})} &\leq C(\delta) \|\G - \GGG\|_{\Cbsp}^p \|\G - \GGG\|_{\Ltwo{\rho_\mathcal{X}}}^{1-p} \nonumber \\
    &\leq C(\delta) (\|\G\|_{\Sp} + \|\GGG\|_{\Rfsp})^p (\xi \|\G\|_{\Sp})^{1-p} \nonumber \\
    &\leq 2^p C(\delta) \xi^{1-p} \|g_{*}\|_{\Sp} \label{c3}
\end{align}
with probability at least $1-\delta$. In the second inequality, we used the fact that
\begin{align}
    \|g\|_{\Cbsp} = \inf \{\|g_1\|_{\Sp} + \|g_2\|_{\Rfsp}\  |\ g=g_1+g_2, g_1 \in \Sp, g_2 \in \Rfsp \}.
\end{align}

Combining \eqref{c4} and \eqref{c3}, we have
\begin{align}
    \|\G - \GG\|_{L^\infty(\rho_{\mathcal{X}})} &\leq \|\G - \GGG\|_{L^\infty(\rho_{\mathcal{X}})} + \|\GGG- \GG\|_{L^\infty(\rho_{\mathcal{X}})} \\
    &\leq \max \left\{ 2^{p+1} C(\delta) \|g_*\|_\Sp \xi^{1-p}, R\left(\frac{2^4 \cdot3L\xi}{\lambda} \|g_*\|_\Sp \right)^{\nicefrac{1}{2}}, \right. \\  &\qquad \qquad\left. R\left(\frac{2^7\cdot3^2L^2\xi^2}{\lambda^3} \mathcal{L}(g_*) \right)^{\nicefrac{1}{4}}, R\left( \frac{2^7\cdot3^2L^3\xi^3}{\lambda^3} \|g_*\|_\Sp\right)^{\nicefrac{1}{4}} \right\}.
\end{align}
As a result, define $\xi>0$ which satisfies 
\begin{multline}
    \xi = \min \left\{ \left(\frac{\epsilon}{2^{p+1} C(\delta) \|g_*\|_\Sp} \right)^{\nicefrac{1}{1-p}}, \frac{\lambda \epsilon^2}{2^4 \cdot3 R^2L\|g_*\|_\Sp},  \left( \frac{\lambda^3 \epsilon^4}{2^7 \cdot3^2 R^4 L^2 \mathcal{L}(g_*)} \right)^{\nicefrac{1}{2}}, \left( \frac{\lambda^3 \epsilon^4}{2^7\cdot 3^2 R^4 L^3 \|g_*\|_\Sp} \right)^{\nicefrac{1}{3}} \right\},
\end{multline}
then we have $\|\G - \GG\|_{L^\infty(\rho_{\mathcal{X}})} \leq \epsilon$ with probability at least $1-2\delta$.
\end{proof}

\section{Proof of Theorem 2}
\label{proof2}
The following theorem shows that if $k$ is a Gaussian kernel and $k_M$ is its random Fourier features approximation, then the norm condition in the assumption is satisfied. The proof is inspired by the analysis of Theorem 4.48 in \cite{steinwart2008support}.
\begin{thm}
Assume $\mathrm{supp}(\rho_\mathcal{X}) \subset \mathbb{R}^d$ is a bounded set and $\rho_{\mathcal{X}}$ has a density with respect to Lebesgue measure which is uniformly bounded away from 0 and $\infty$ on $\mathrm{supp}(\rho_\mathcal{X})$. Let $k$ be a Gaussian kernel and $\Sp$ be its RKHS, then for any $m\geq d/2$, there exists a constant $C_{m, d}>0$ such that 
\begin{equation}
    \| f \|_{L^\infty(\rho_{\mathcal{X}})} \leq C_{m, d} \|f\|_{\Sp}^{\nicefrac{d}{2m}} \|f\|_{\Ltwo{\rho_\mathcal{X}}}^{1-\nicefrac{d}{2m}}
\end{equation}
for any $f \in \Sp$. Also, for any $M\geq 1$, let $k_M$ be a random Fourier features approximation of $k$ with $M$ features and $\Cbsp$ be a RKHS of $k+k_M$. Then with probability at least $1-\delta$ with respect to a sampling of features, 
\begin{equation}
    \| f \|_{L^\infty(\rho_{\mathcal{X}})} \leq C_{m, d} \left(1 + \frac{1}{\delta} \right)^{\nicefrac{d}{4m}} \|f\|_{\Cbsp}^{\nicefrac{d}{2m}} \|f\|_{\Ltwo{\rho_\mathcal{X}}}^{1-\nicefrac{d}{2m}}
\end{equation}
for any $f \in \Cbsp$.
\end{thm}

\begin{proof}
For notational simplicity, we denote $\mathrm{supp}(\rho_{\mathcal{X}})$ by $\mathcal{X}^\prime$. From the boundedness of $\mathcal{X}^\prime$ and the condition on $\rho_\mathcal{X}$, the following relation holds for any $f \in L^\infty(\rho_\mathcal{X})$:
\begin{align}
    \|f\|_{L^\infty(\rho_\mathcal{X})} = \|f\|_{L^\infty(\mathcal{X}^\prime)} \label{inf1}\\
    \|f\|_{\Ltwo{\rho_\mathcal{X}}} \geq C_1 \|f\|_{L^2({\mathcal{X}^\prime})},\label{inf2}
\end{align}
where $C_1>0$ is a constant. From the discussion after Theorem \ref{thm2}, for any $f \in W^m{(\mathcal{X}^\prime)}\ (m\geq d/2)$ there exists a constant $C_2>0$ such that the following inequality holds:
\begin{align}
    \| f \|_{L^\infty(\mathcal{X}^\prime)} \leq C_2 \|f\|_{W^m{(\mathcal{X}^\prime)}}^{\nicefrac{d}{2m}}  \|f\|_{L^2({\mathcal{X}^\prime})}^{1-\nicefrac{d}{2m}}.\label{sob1}
\end{align}
Here $W^m(\mathcal{X}^\prime)$ is Sobolev space with order $m$ defined as follows:
\begin{equation}
    W^m{(\mathcal{X}^\prime)} = \left\{ f \in L^2(\mathcal{X}^\prime) \ \middle|\  \partial^{(\alpha)}f \in  L^2(\mathcal{X}^\prime) \text{ exists for all } \alpha \in \mathbb{N}^d \text{ with } |\alpha| \leq m \right\}, 
\end{equation}
where $\partial^{(\alpha)}$ is the $\alpha$-th weak derivative for a multi-index $\alpha = (\alpha^{(1)}, \ldots, \alpha^{(d)}) \in \mathbb{N}^d$ with $|\alpha| = \sum_{i=1}^d \alpha^{(i)}.$\\
Combining \eqref{inf1}, \eqref{inf2} and \eqref{sob1}, we have
\begin{align}
    \| f \|_{L^\infty(\rho_{\mathcal{X}})} \leq C \|f\|_{W^m{(\mathcal{X}^\prime)}}^{\nicefrac{d}{2m}}  \|f\|_{L^2(\rho_\mathcal{X})}^{1-\nicefrac{d}{2m}},\label{sob2}
\end{align}
where $C>0$ is a constant. So it suffices to show that $\Sp$ and $\Cbsp$ can be continuously embedded in $W^m({\mathcal{X}^\prime})$. 
For $\Sp$, it can be shown in the same manner as Theorem 4.48 in \cite{steinwart2008support}. For $\Cbsp$, we first define a spectral measure of the kernel function $k+k_M$ as
\begin{equation}
    \tau^+(\omega) = \frac{1}{M} \sum_{i=1}^M \delta(\omega - \omega_i) + \tau(\omega),
\end{equation}
where $\delta$ is a Dirac measure on $\Omega$. Then a kernel function $k+k_M$ can be written as
\begin{equation}
    (k + k_M)(x, x^\prime) = \int_\Omega \varphi(x, \omega) \overline{\varphi(x^\prime, \omega)} \mathrm{d}\tau^+(\omega),
\end{equation}
and from \cite{bach2017equivalence}, for any $f \in \Cbsp$,
there exists $g \in \Ltwo{\tau^+}$ such that
\begin{align}
    &f(x) = \int_\Omega g(\omega) \varphi(x, \omega) \mathrm{d}\tau^+(\omega),\\
     &\| f \|_{\Cbsp} = \|g\|_{\Ltwo{\tau^+}}.
\end{align}
Let us fix a multi-index $\alpha = (\alpha^{(1)}, \ldots, \alpha^{(d)}) \in \mathbb{N}^d$ and $|\alpha| = m.$ For $\alpha \in \mathbb{N}^d$, we write $\partial^\alpha = \partial_1^{\alpha^{(1)}}\cdots \partial_d^{\alpha^{(d)}}$. We then have
\begin{align}
    \| \partial^\alpha f \|_{L^2(\mathcal{X}^\prime)}^2 &= \int_{\mathcal{X}^\prime} \left( \partial_x^\alpha \int_\Omega g(\omega) \varphi(x, \omega) \mathrm{d}\tau^+(\omega) \right)^2 \mathrm{d}x \\
    &\leq \int_{\mathcal{X}^\prime} \left(\int_\Omega |g(\omega)| \partial_x^\alpha \varphi(x, \omega) d\tau^+(\omega) \right)^2 \mathrm{d}x \\
    &\leq \| g \|_{\Ltwo{\tau^+}}^2 \int_{\mathcal{X}^\prime} \int_\Omega  |\partial_x^\alpha \varphi(x, \omega)|^2 \mathrm{d}\tau^+(\omega)  \mathrm{d}x.
\end{align}
Because we consider $\varphi$ as a random Fourier feature, $\Omega = \mathbb{R}^d$ and 
\begin{align}
    \varphi(x, \omega) &= C^\prime e^{i \omega^\top x}, \\
    \partial_x^\alpha \varphi(x, \omega) &= \omega^\alpha C^\prime e^{i \omega^\top x}
\end{align}
where $C^\prime>0$ is a normalization constant and $\omega^\alpha = \prod_{i=1}^d {\omega^{(i)}}^{\alpha_i}$ for $\omega = (\omega^{(1)}, \ldots, \omega^{(d)}) \in \mathbb{R}^d$ and $\alpha = (\alpha^{(1)}, \ldots, \alpha^{(d)}) \in \mathbb{N}^d$. So we have
\begin{align}
    \| \partial^\alpha f \|_{L^2(\mathcal{X}^\prime)}^2 &\leq \| g \|_{\Ltwo{\tau^+}}^2 \int_{\mathcal{X}^\prime} C^{\prime 2}\int_\Omega\omega^{2\alpha} \mathrm{d}\tau^+(\omega)  \mathrm{d}x \\
    &\leq  C^{\prime 2}  \mathrm{vol}(\mathcal{X}^\prime)\| f \|_{\Cbsp}^2 \left( \mathbb{E}_{\omega \sim \tau}\left[\omega^{2\alpha}\right] + \frac{1}{M} \sum_{i=1}^M \omega_i^{2\alpha} \right).
\end{align}
We note that because $\tau$ is Gaussian, $\mathbb{E}_{\omega \sim \tau}\left[\omega^{2\alpha}\right]$ is finite for any $\alpha \in \mathbb{N}^d$. Because $\omega_i \sim \tau$ and $\omega_i^{2\alpha}$ is non negative, from Markov's inequality we have
\begin{equation}
    \frac{1}{M} \sum_{i=1}^M \omega_i^{2\alpha} \leq \frac{1}{\delta} \mathbb{E}_{\omega \sim \tau}\left[\omega^{2\alpha}\right]
\end{equation}
with probability at least $1-\delta$. 
As a result, we have
\begin{align}
    \| \partial^\alpha f \|_{L^2({\mathcal{X}^\prime})}^2 &\leq \left(1+\frac{1}{\delta}\right) C^{\prime 2} \mathrm{vol}(\mathcal{X}^\prime) \| f \|_{\Cbsp}^2 \mathbb{E}_{\omega \sim \tau} \left[\omega^{2\alpha}\right].
\end{align}
So we can compute Sobolev norms of $f$ as follows:
\begin{align}
    \| f \|_{W^m{(\mathcal{X}^\prime)}}^2 &= \sum_{|\alpha|\leq m} \| \partial^\alpha f \|_{L^2({\mathcal{X}^\prime})}^2 \\
    &\leq \left(1+\frac{1}{\delta}\right) C^{\prime 2} \mathrm{vol}(\mathcal{X}^\prime) \| f \|_{\Cbsp}^2 \sum_{|\alpha|\leq m} \mathbb{E}_{\omega \sim \tau} \left[\omega^{2\alpha}\right]. \label{sob3}
\end{align}
Substitute \eqref{sob3} to \eqref{sob2} and define $C_{m, d} = C \left( C^{\prime 2} \mathrm{vol}(\mathcal{X}^\prime) \sum_{|\alpha|\leq m} \mathbb{E}_{\omega \sim \tau} \left[\omega^{2\alpha}\right]\right)^{d/4m}$, we get a desired result.
\end{proof}

\begin{rmk}
We note that the assumption that $k$ is Gaussian is only used to derive $\mathbb{E}_{\omega \sim \tau} \left[\omega^{2\alpha} \right]$ is finite for all $\alpha \in \mathbb{N}^d.$ This means that if $\psi(x-y)=k(x-y)$ belongs to Schwartz class (a space of rapidly decreasing function) \cite{yoshida1995functional}, its Fourier transform $\tau$ also belongs to this class, thus the above finite moment property is satisfied.
\end{rmk}

\section{Proof of Theorem 3}
In this section, we provide the complete statement and the proof of Theorem \ref{thm3}. First, we provide some useful propositions which are appeared in \cite{nitanda2018stochastic}. \par
The first proposition suggests that there exists a sufficiently small $\lambda>0$ such that $g_\lambda$ is also the Bayes classifier.
\begin{prop}[Proposition A in \cite{nitanda2018stochastic}]
\label{prop3}
Suppose Assumption \ref{as2-2}, \ref{as2-5}, \ref{as2-6}, \ref{as2-7} hold. Then, there exists $\lambda>0$ such that $\|g_\lambda - g_*\|_{L^\infty(\rho_{\mathcal{X}})} \leq m(\delta) / 2$.
\end{prop}
The second proposition shows that the distance between expected estimator ${E}[g_{T+1}]$ and the population risk minimizer $g_{M, \lambda}$ converges sub-linearly.
\begin{prop}[Modified version of Proposition C in \cite{nitanda2018stochastic}]
\label{prop4}
Suppose Assumption \ref{as2-3}, \ref{as2-8} holds. Consider Algorithm \ref{alg1} with $\eta_t = \frac{2}{\lambda(\gamma+t)}$ and $\alpha_t = \frac{2(\gamma+t-1)}{(2\gamma+T)(T+1)}$ and assume assume $\|\overline{g}_1\|_{\Sp} \leq (2\gamma_1+1/\lambda)GR$ and $\eta_1 \leq \min \{1/L, 1/2\lambda \}.$ Then, it follows that
\begin{equation}
    \| \mathbb{E}[\overline{g}_{T+1}] - g_{\lambda} \|_{\Sp}^2 \leq \frac{2}{\lambda} \left( \frac{18G^2R^2}{\lambda(2\gamma+T)} + \frac{\lambda \gamma(\gamma-1)}{2(2\gamma+T)(T+1)} \|\overline{g}_1 - g_{\lambda} \|_{\Sp}^2 \right).
\end{equation}
\end{prop}
The last proposition is about the concentration of the estimator around its mean.
\begin{prop}[Modified version of Proposition 2 and D in \cite{nitanda2018stochastic}]
\label{prop5}
Suppose Assumption \ref{as2-1}, \ref{as2-3} and \ref{as2-8} holds. Consider Algorithm \ref{alg1} with $\eta_t = \frac{2}{\lambda(\gamma+t)}$ and $\alpha_t = \frac{2(\gamma+t-1)}{(2\gamma+T)(T+1)}$ and assume $\|\overline{g}_1\|_{\Sp} \leq (2\gamma_1+1/\lambda)GR$ and $\eta_1 \leq \min \{1/L, 1/2\lambda \}.$ Then, it follows that
\begin{equation}
    \mathbb{P} \left[\left\| \overline{g}_{T+1} - \mathbb{E}[\overline{g}_{T+1}] \right\|_{\Sp} \geq \epsilon \right] \leq 2 \exp \left(-\frac{\lambda^2(2\gamma+T)}{2^6 \cdot 3^2 G^2 R^2} \epsilon^2 \right).
\end{equation}
\end{prop}
\begin{rmk}
We note that in \cite{nitanda2018stochastic}, they assumed only the Lipschitz smoothness of $\mathcal{L}(g)$ with respect to $\|\cdot\|_\Sp$-norm, but they used the Lipschitz smoothness of $l(g,z)$ with respect to $\|\cdot\|_\Sp$-norm in the proof of Proposition B. Thus we deal with the Lipschitz smoothness of $l(\cdot, y)$ with respect to the first variable instead (Assumption \ref{as2-8}) and correct these proofs.
\end{rmk}
\label{proof3}

Using these propositions, our main result about the exponential convergence of the expected classification error is shown as follows.
\begin{thm}
Suppose Assumptions \ref{as2-1}-\ref{as2-7} holds. There exists a sufficiently small $\lambda>0$ such that the following statement holds:\\
Taking the number of random features $M$ that satisfies

\begin{equation}
\label{thm3-1}
    M \geq \max \left\{ \frac{8}{3} \left( \frac{R}{\xi} \right)^2, 32 \left( \frac{R}{\xi} \right)^4 \right\} \log \frac{2R^2}{\| T \|_{\mathrm{op}} \delta}
\end{equation}
where $\xi > 0$ is defined as below:
\begin{multline}
    \xi = \min \left\{ \left(\frac{m(\delta)}{2^{p+3} C(\delta^\prime) \|g_*\|_\Sp} \right)^{\nicefrac{1}{1-p}}, \frac{\lambda m^2(\delta)}{2^8 \cdot 3 R^2G\|g_*\|_\Sp}, \right. \\ \left. \left( \frac{\lambda^3 m^4(\delta)}{2^{15}\cdot 3^2 R^4 G^2 \mathcal{L}(g_*)} \right)^{\nicefrac{1}{2}}, \left( \frac{\lambda^3 m^4(\delta)}{2^{15} \cdot 3^2 R^4 G^3 \|g_*\|_\Sp} \right)^{\nicefrac{1}{3}} \right\}.
\end{multline}
Consider Algorithm \ref{alg1} with $\eta_t = \frac{2}{\lambda(\gamma+t)}$ and $\alpha_t = \frac{2(\gamma+t-1)}{(2\gamma+T)(T+1)}$ where $\gamma$ is a positive value such that $\|g_1\|_{\Rfsp} \leq (2\eta_1+1/\lambda)GR$ and $\eta_1 \leq \min \{1/L, 1/2\lambda \}.$ 
Then, with probability $1-2\delta^\prime$, for sufficiently large $T$ such that
\begin{equation}
    \max \left\{ \frac{36G^2R^2}{\lambda^2(2\gamma+T)}, \frac{\gamma (\gamma-1) \|g_1 - g_{M, \lambda} \|_{\Rfsp}^2}{(2\gamma+T)(T+1)} \right \} \leq \frac{m^2(\delta)}{64R^2},
\end{equation}
we have the following inequality for any $t>T$:
\begin{align}
    \mathbb{E} \left[ \mathcal{R}(\overline{g}_{T+1}) - \mathcal{R}(\mathbb{E}[Y|x]) \right]
    \leq 2 \exp \left( -\frac{\lambda^2 (2\gamma + t) m^2(\delta)}{2^{12} \cdot 9 G^2R^4}  \right).
\end{align}
\end{thm}

\begin{proof}
Fix $\lambda>0$ satisfying the condition in Proposition \ref{prop3}. From Theorem \ref{thm1}, if we set a number of features $M$ satisfying \eqref{thm3-1}, we have
\begin{align}
    \|g_{M, \lambda} - g_* \|_{L^\infty(\rho_{\mathcal{X}})} &\leq \|g_{M, \lambda} - g_\lambda \|_{L^\infty(\rho_{\mathcal{X}})} + \|g_{\lambda} - g_* \|_{L^\infty(\rho_{\mathcal{X}})} \\
    &\leq \frac{m(\delta)}{4} + \frac{m(\delta)}{2} = \frac{3m(\delta)}{4}.
\end{align}
Then $\sgn(g(X)) = \sgn(g_*(X))$ almost surely for any $g\in \Rfsp$ satisfying $\|g - g_{M, \lambda} \|_{\Rfsp} \leq m(\delta) / 4R$, since
\begin{align}
    \|g - g_* \|_{L^\infty(\rho_{\mathcal{X}})} &\leq \|g - g_{M, \lambda} \|_{L^\infty(\rho_{\mathcal{X}})} + \|g_{M, \lambda} - g_* \|_{L^\infty(\rho_{\mathcal{X}})} \\
    &\leq R\|g - g_{M, \lambda} \|_{\Rfsp} + \|g_{M, \lambda} - g_* \|_{L^\infty(\rho_{\mathcal{X}})} \\
    &\leq \frac{m(\delta)}{4} + \frac{3m(\delta)}{4} = m(\delta)
\end{align}
and $|g_*(X)| \geq m(\delta)$ almost surely. In other words, $g$ is also the Bayes classifier of $\mathcal{R}(g).$ 
Assume 
\begin{equation}
    \| \mathbb{E}[\overline{g}_{T+1}] - g_{M, \lambda} \|_{\Rfsp} \leq \frac{m(\delta)}{8R}. \label{thm3-2}
\end{equation}
Then, substituting $\epsilon = m(\delta) / 8R$ in Proposition \ref{prop5}, we have
\begin{equation}
    \left\| \overline{g}_{T+1} - g_{M, \lambda} \right\|_{\Rfsp} \leq \left\| \overline{g}_{T+1} - \mathbb{E}[\overline{g}_{T+1}] \right\|_{\Rfsp} + \| \mathbb{E}[\overline{g}_{T+1}] - g_{M, \lambda} \|_{\Rfsp}  \leq \frac{m(\delta)}{4R}
\end{equation}
with probability at least $1-2\exp \left(-\frac{\lambda^2(2\gamma+T) m^2(\delta)}{2^{12} \cdot 3^2 G^2 R^4} \right)$. In other words, $\overline{g}_{T+1}$ is also the Bayes classifier with same probability. By definition of the expected classification error, we have
\begin{equation}
    \mathbb{E}[\mathcal{R}(\overline{g}_{T+1})] - \mathcal{R}(\mathbb{E}[Y|x]) \leq 1-2\exp \left(-\frac{\lambda^2(2\gamma+T) m^2(\delta)}{2^{12} \cdot 3^2 G^2 R^4} \right).
\end{equation}
Finally, to satisfy \eqref{thm3-2}, the required number of iteration $T$ is obtained by Proposition \ref{prop4}, which completes the proof.
\end{proof}

\section{Proof of Corollary 1}
\label{proof4}
Although $g_\lambda$ converges to $g_*$ as $\lambda \rightarrow 0$ as shown in Proposition \ref{prop3}, specifying its convergence rate is difficult in general. To derive its rate, first we need the \textit{local strong convexity}, which is a strong convexity on a arbitrary compact set.
\begin{asm}
\label{as3-7}
$\phi:\mathbb{R} \rightarrow \mathbb{R}$ is $\mu(U)$-strongly convex on a bounded set $[-U, U]\subset \mathbb{R}$, i.e., 
\begin{align}
    \phi(\zeta_1) - \phi(\zeta_2) - \phi^\prime(\zeta_2) (\zeta_1-\zeta_2) \geq \frac{\mu(U)}{2} (\zeta_1-\zeta_2)^2.
\end{align}
holds for any $\zeta_1,\zeta_2\in[-U,U]$.
\end{asm}
\begin{lem}
\label{lem2}
Assume $\mathrm{supp}(\rho_\mathcal{X}) \subset \mathbb{R}^d$ is a bounded set and $\rho_{\mathcal{X}}$ has a density with respect to Lebesgue measure, which is uniformly bounded away from 0 and $\infty$ on $\mathrm{supp}(\rho_\mathcal{X})$. 
Let $k$ be a Gaussian kernel and $l$ satisfies Assumption \ref{as3-7}. Then for arbitrary small $\epsilon>0$, there exists a constant $C>0$ such that
\begin{align}
    \|g_\lambda - g_*\|_{L^\infty(\rho_{\mathcal{X}})} \leq C \|g_*\|_\Sp \left( \frac{\lambda}{\mu(R\|g_*\|_\Sp)} \right)^{\frac{1}{2}-\epsilon}.
\end{align}

\end{lem}

\begin{proof}
By definition of $g_\lambda$, we have
\begin{align}
    &\mathcal{L}(g_*) + \frac{\lambda}{2} \|g_*\|_\Sp^2 \geq \mathcal{L}(g_\lambda) + \frac{\lambda}{2} \|g_\lambda\|_\Sp^2, \label{eqap3-1}\\
    &\|g_*\|_\Sp \geq \|g_\lambda\|_\Sp.\label{eqap3-2}
\end{align}
In addition, it holds that
\begin{align}
    \label{eqap3-4}
    &g_*(x) \leq R\|g_*\|_\Sp, \\
    &g_\lambda(x) \leq R\|g_\lambda\|_\Sp \leq R\|g_*\|_\Sp
\end{align}
for all $x \in \mathcal{X}$. Furthermore, since $g_*$ attains infimum of $\mathcal{L}$ among all measurable functions, we have
\begin{align}
    \int_{\mathcal{Y}} \partial_\zeta l(g_*(\cdot),y) \mathrm{d}\rho(y|\cdot) \equiv 0, \label{eqap3-3}
\end{align}
where $\partial_\zeta$ denotes a partial derivative of $l$ with respect to the first variable.\\
Then we obtain
\begin{align}
    \|g_\lambda - g_*\|^2_{L^2(\rho_{\mathcal{X}})} &= \int_{\mathcal{X}} \left|g_\lambda(x) - g_*(x)\right|^2 \mathrm{d}\rho_\mathcal{X}(x) \\
    &\leq \int_{\mathcal{X} \times \mathcal{Y}} \frac{2}{\mu(R\|g_*\|_\Sp)} \{ l(g_\lambda(x),y) - l(g_*(x),y)  \\
    & \qquad- \partial_\zeta l(g_*(x),y)(g_\lambda(x) - g_*(x)) \}  \mathrm{d}\rho(x, y) \quad (\because \text{\eqref{eqap3-4} and Assumption \ref{as3-7}})\\
    &= \int_{\mathcal{X} \times \mathcal{Y}} \frac{2}{\mu(R\|g_*\|_\Sp)}\left\{ l(g_\lambda(x),y) - l(g_*(x),y) \right\}\mathrm{d}\rho(x, y)\quad(\because \eqref{eqap3-3})\\
    &= \frac{2}{\mu(R\|g_*\|_\Sp)} \left(\mathcal{L}(g_\lambda) -  \mathcal{L}(g_*)\right)\\
    &\leq \frac{\lambda}{\mu(R\|g_*\|_\Sp)}\left(\|g_*\|_\Sp^2 - \|g_\lambda\|_\Sp^2\right) \quad(\because \eqref{eqap3-1})\\
    &\leq \frac{\lambda}{\mu(R\|g_*\|_\Sp)} \|g_*\|_\Sp^2 \quad (\because \eqref{eqap3-2}).
\end{align}
Finally, applying the first part of Theorem \ref{thm2} with $p=d/2m$, we obtain
\begin{align}
    \|g_\lambda - g_*\|_{L^\infty(\rho_{\mathcal{X}})} &\leq C_p \|g_\lambda - g_*\|_\Sp^p \|g_\lambda - g_*\|_{L^2(\rho_{\mathcal{X}})}^{1-p} \\
    &\leq 2^p C_p \left(\frac{\lambda}{\mu(R\|g_*\|_\Sp)}\right)^{\frac{1-p}{2}} \|g_*\|_\Sp
\end{align}
for any $0<p<1$ and get a desired result.
\end{proof}

\begin{cor}
Assume $\mathrm{supp}(\rho_\mathcal{X}) \subset \mathbb{R}^d$ is a bounded set and $\rho_{\mathcal{X}}$ has a density with respect to Lebesgue measure, which is uniformly bounded away from 0 and $\infty$ on $\mathrm{supp}(\rho_\mathcal{X})$. 
Let $k$ be a Gaussian kernel and $l$ be logistic loss. Under Assumption \ref{as2-5}-\ref{as2-7}, the following statement holds:\\
Taking a regularization parameter $\lambda$ and a number of random features $M$ that satisfies
\begin{align}
    &\lambda \lesssim \log^3 \frac{1+2\delta}{1-2\delta} \cdot \frac{1}{ (2+e^{R\|g_*\|_\Sp}+e^{-R\|g_*\|_\Sp})\|g_*\|_\Sp^3},\\
        &M \gtrsim \left(\frac{\left(1+\frac{1}{\delta^\prime}\right)  \|g_*\|_{\mathcal{H}}^4}{\lambda^3 \log^4 \frac{1+2\delta}{1-2\delta}}\right)^2 \log \frac{1}{\delta^\prime}.
\end{align}
Consider Algorithm \ref{alg1} with $\eta_t = \frac{2}{\lambda(\gamma+t)}$ and $\alpha_t = \frac{2(\gamma+t-1)}{(2\gamma+T)(T+1)}$ where $\gamma$ is a positive value such that $\|g_1\|_{\Rfsp} \leq (2\eta_1+1/\lambda)$ and $\eta_1 \leq \min \{4, 1/2\lambda \}.$ Then, with probability $1-2\delta^\prime$, for a sufficiently large $T$ such that
\begin{align}
\max \left\{ \frac{36}{\lambda^2(2\gamma+T)}, \frac{\gamma (\gamma-1) \|g_1 - g_{M, \lambda} \|_{\Rfsp}^2}{(2\gamma+T)(T+1)} \right \} \leq \frac{\log^2 \frac{1+2\delta}{1-2\delta}}{64},
\end{align}
we have the following inequality for any $t\geq T$:
\begin{align}
    \mathbb{E} \left[ \mathcal{R}(\overline{g}_{t+1}) - \mathcal{R}(\mathbb{E}[Y|x]) \right] 
   \leq 2 \exp \left( -\frac{\lambda^2 (2\gamma + t)}{2^{12} \cdot 9} \log^2 \frac{1+2\delta}{1-2\delta} \right).
\end{align}
\end{cor}

\begin{proof}
When $l$ is logistic loss, we have $\phi(v) = \log(1+\exp(-v))$ and $\phi^{\prime \prime}(v) = \frac{1}{2+e^v+e^{-v}}$. Thus it follows that Assumption \ref{as3-7} is satisfied with $\mu(U) = \frac{2}{1+e^{-U}+e^U}.$ To satisfy the condition 
\begin{align}
    \|g_\lambda - g_*\|_{L^\infty(\rho_{\mathcal{X}})} \leq m(\delta) / 2,
\end{align}
required $\lambda$ is easily derived from Lemma \ref{lem2} with, for example,  $\epsilon=1/6$. In addition, since $\phi^{\prime \prime}(v) \leq 1/4$ and $\phi^\prime(v) \leq 1$ for any $v \in \mathbb{R}$, Assumption \ref{as2-8} and Assumption \ref{as2-1} are satisfied with $L=1/4$ and $G=1$, respectively.  Substituting them and $m(\delta)=\log((1+2\delta)/(1-2\delta))$, $R=1$ in Theorem \ref{thm3}, we get a desired result.
\end{proof}

\end{document}